\documentclass{article}

\PassOptionsToPackage{numbers, compress}{natbib}

\usepackage[final]{neurips_2024}

\usepackage[utf8]{inputenc} 
\usepackage[T1]{fontenc}    
\usepackage{hyperref}       
\usepackage{url}            
\usepackage{booktabs}       
\usepackage{amsfonts}       
\usepackage{nicefrac}       
\usepackage{microtype}      
\usepackage{xcolor}         
\usepackage{wrapfig}

\usepackage[utf8]{inputenc} 
\usepackage[T1]{fontenc}    
\usepackage{hyperref}       
\usepackage{url}            
\usepackage{booktabs}       
\usepackage{amsfonts}       
\usepackage{nicefrac}       
\usepackage{microtype}      
\usepackage{xcolor}         
\usepackage{amsmath,amssymb,mathtools}
\usepackage{amsthm} 
\usepackage{algorithm}
\usepackage[noend]{algcompatible}
\usepackage{tikz}
\usetikzlibrary{positioning}
\usetikzlibrary{shapes, arrows, arrows.meta, positioning}
\usepackage{caption}
\usepackage{subcaption}
\usepackage{mleftright}
\usepackage{graphicx}
\usepackage{multirow}
\usepackage{enumitem}
\usepackage{comment}
\usepackage{xspace}
\usepackage{float}
\usepackage{cleveref}

\tikzstyle{block} = [draw, fill=white, circle, text centered,inner sep=0.25cm]
\tikzstyle{block-s} = [draw, fill=white, circle, double, double distance=1pt, text centered,inner sep=0.25cm]

\tikzstyle{roundnode} = [circle, draw=black!60,, fill=white, thick, inner sep = 1pt]
\tikzstyle{dashednode} = [circle, draw=black!60, dashed, fill=white, thick, inner sep=1pt]
\tikzset{edge/.style = {->,> = latex',-{Latex[width=1.5mm]}}}
\tikzset{edge2/.style={Latex-Latex,dashed}}

\theoremstyle{plain}
\newtheorem{theorem}{Theorem}[section]

\newtheorem{lemma}[theorem]{Lemma}
\newtheorem{corollary}[theorem]{Corollary}
\theoremstyle{definition}
\newtheorem{definition}[theorem]{Definition}

\newtheorem{remark}[theorem]{Remark}

\newtheorem{example}{Example}

\newcommand{\eqdef}{\coloneqq}

\newcommand{\Gr}{\mathcal{G}}
\newcommand{\M}{\mathcal{M}}

\newcommand{\mysum}[2]{\sum\limits_{#1} ^{ #2} }

\newcommand{\independent}{\perp\!\!\!\perp}
\newcommand{\notindependent}{\not\!\perp\!\!\!\perp}

\newcommand{\Pa}[2]{\textit{Pa}_{#2}(#1)}

\newcommand{\Anc}[2]{\textit{Anc}_{#2}(#1)}

\newcommand{\Vb}{\mathbf{V}}
\newcommand{\Xb}{\mathbf{X}}
\newcommand{\Yb}{\mathbf{Y}}

\newcommand{\Wb}{\mathbf{W}}
\newcommand{\Hb}{\mathbf{H}}
\newcommand{\Zb}{\mathbf{Z}}
\newcommand{\Ub}{\mathbf{U}}
\newcommand{\Cb}{\mathbf{C}}
\newcommand{\Fb}{\mathbf{F}}

\newcommand{\Db}{\mathbf{D}}
\newcommand{\Ab}{\mathbf{A}}
\newcommand{\Tb}{\mathbf{T}}
\newcommand{\Eb}{\mathbf{E}}

\newcommand{\pro}{P}
\newcommand{\interSbp}[1][\Yb]{\prs_{\Xb}(#1)}

\newcommand{\s}[0]{\textsc{s}}
\newcommand{\cc}[0]{\textsc{c}}

\newcommand{\Qs}[0]{Q^\s}
\newcommand{\cco}[0]{\cc-component\xspace}
\newcommand{\ccos}[0]{\cc-components\xspace}
\newcommand{\sco}[0]{\s-component\xspace}
\newcommand{\scos}[0]{\s-components\xspace}
\newcommand{\Xs}[0]{\Xb_{\textsc{as}}}
\newcommand{\Xns}[0]{\Xb_{\textsc{ns}}}
\newcommand{\Ys}[0]{\Yb_{\textsc{as}}}
\newcommand{\Yns}[0]{\Yb_{\textsc{ns}}}
\newcommand{\Vs}[0]{\Vb_{\textsc{as}}}
\newcommand{\Vns}[0]{\Vb_{\textsc{ns}}}

\newcommand{\Grs}{\mathcal{G}^{\s}}
\newcommand{\prs}{P^{\s}}

\newcommand{\redcross}{\textcolor{red}{\scalebox{1}{$\times$}}}

\title{Causal Effect Identification in a Sub-Population with Latent Variables}

\author{%
    Amir Mohammad Abouei \textsuperscript{\rm 1}, Ehsan Mokhtarian \textsuperscript{\rm 1}, Negar Kiyavash \textsuperscript{\rm 2}, Matthias Grossglauser \textsuperscript{\rm 1}  \\
    \AND
    \\
    \textsuperscript{\rm 1}School of Computer and Communication Sciences, EPFL \\
    \textsuperscript{\rm 2}College of Management of Technology, EPFL\\
    \{amir.abouei, ehsan.mokhtarian, negar.kiyavash, matthias.grossglauser\}@epfl.ch \\
}

\begin{document}
\maketitle
\begin{abstract}
    The \s-ID problem seeks to compute a causal effect in a specific sub-population from the observational data pertaining to the same sub-population \cite{abouei2023sid}. This problem has been addressed when all the variables in the system are observable. In this paper, we consider an extension of the \s-ID problem that allows for the presence of latent variables. To tackle the challenges induced by the presence of latent variables in a sub-population, we first extend the classical relevant graphical definitions, such as \cc-components and Hedges, initially defined for the so-called ID problem \cite{pearl1995causal, tian2002general}, to their new counterparts. Subsequently, we propose a sound algorithm for the \s-ID problem with latent variables.
\end{abstract}

\section{Introduction} \label{sec: intro}

    Causal inference, i.e., understanding the effect of an intervention in a stochastic system, is a key focus of research in statistics and machine learning \cite{rubin1974estimating, pearl2000models, pearl2009causality, spirtes2000causation}.
    Scientists, policymakers, business leaders, and healthcare professionals must understand causal relationships to move beyond correlations and make informed, evidence-based decisions.
    To perform causal inference tasks, it is crucial to differentiate between two types of data: observational and interventional \cite{pearl2018book}.

    \begin{figure}[h!]
        \centering
        \begin{subfigure}[b]{0.48\textwidth}
            \centering
            \includegraphics[width=\textwidth]{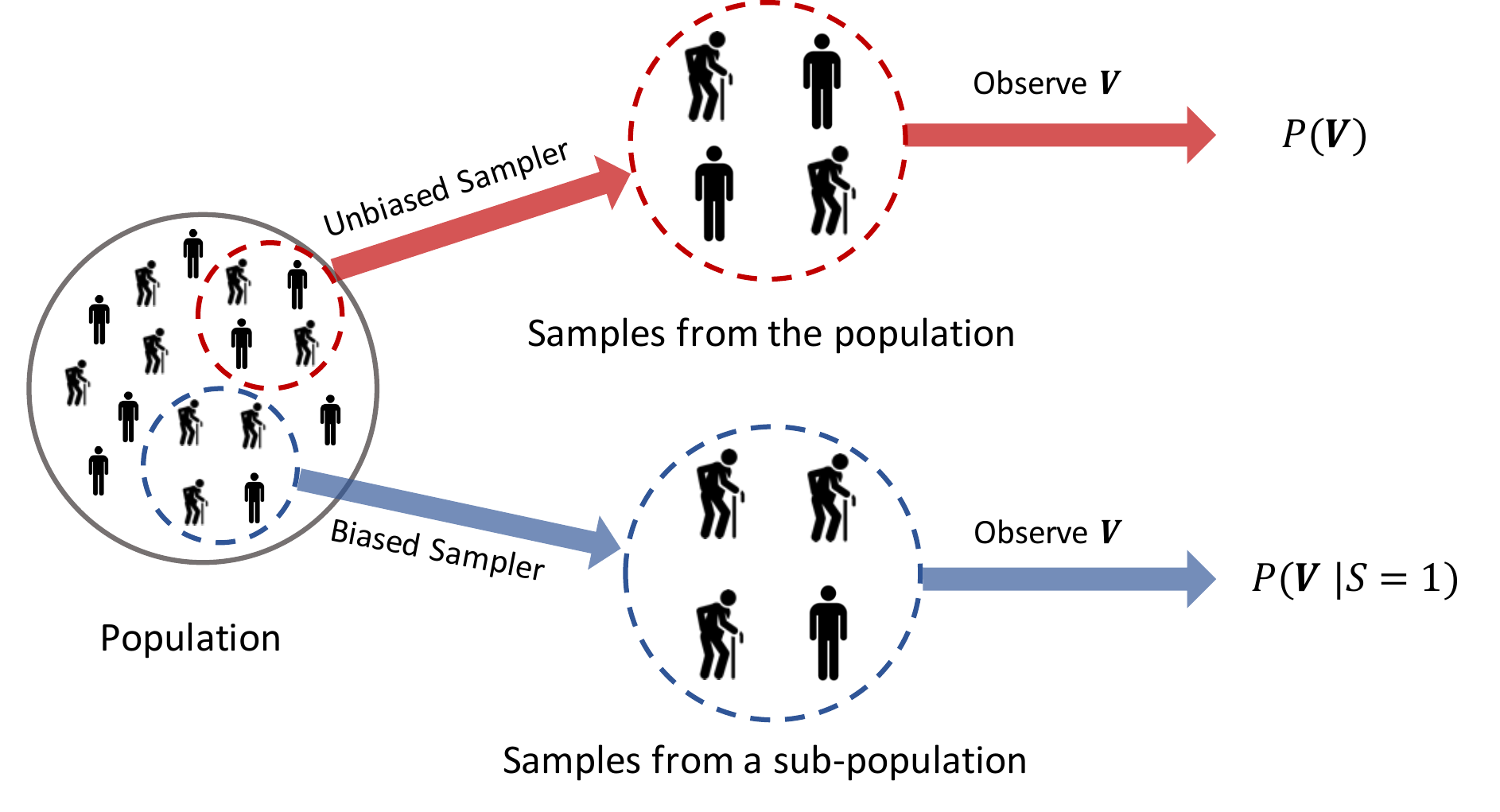}
            \caption{Observational data.}
            \label{fig: intro 1}
        \end{subfigure}
        \begin{subfigure}[b]{0.48\textwidth}
            \centering
            \includegraphics[width=\textwidth]{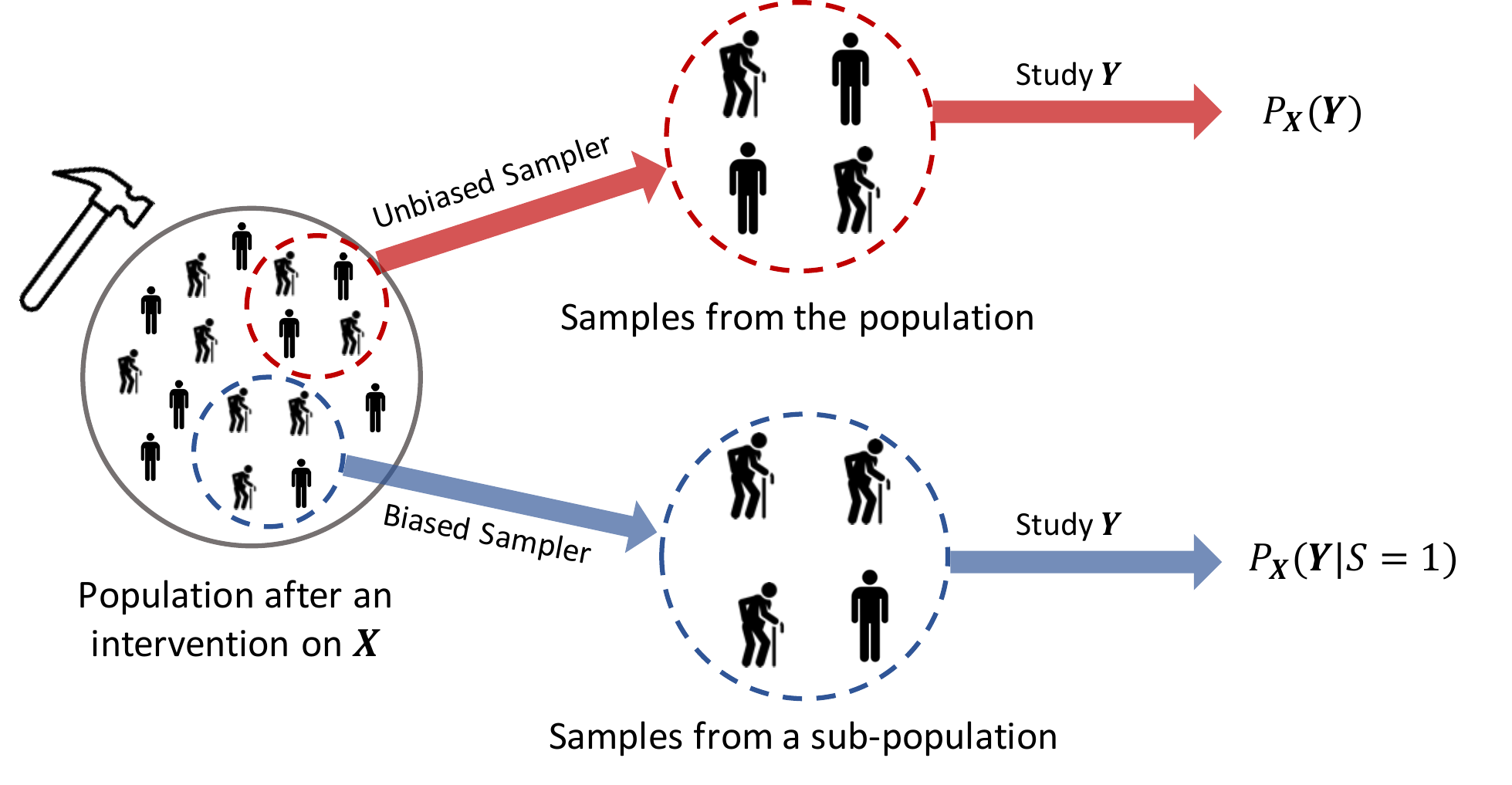}
            \caption{Interventional data.}
            \label{fig: intro 2}
        \end{subfigure}
    
        \caption{A population consists of a sample space for the study of the causal effect of an intervention. While the unbiased sampler draws samples uniformly at random, the biased sampler selects samples based on certain criteria, forming a sub-population.}
        \label{fig: intro}
    \end{figure}

    \noindent \textbf{Observational Data.}
    Figure \ref{fig: intro} illustrates a \emph{population} that pertains to the entire \emph{sample space} for a study of the causal effect of some intervention.
    A \emph{sampler} draws samples from the population.
    The sampler is unbiased if it draws samples at random such that each individual in the population has an equal chance of being selected.
    As a result, the obtained sample is representative of the entire population.
    In contrast, a biased sampler selects samples based on certain criteria forming a \emph{sub-population}.
    For each extracted sample, we collect data from a set of observed features denoted by $\Vb$.
    As depicted in Figure \ref{fig: intro 1}, when the sampler is unbiased, the observational data comes from the joint distribution $\pro(\Vb)$.
    For a biased sampler, the observations can be modeled as drawn from a conditional distribution $\pro(\Vb \vert S=1)$, where $S=1$ indicates that the sample belongs to a sub-population.

    \noindent  \textbf{Interventional Data.}
    An \emph{intervention} on a subset $\Xb \subseteq \Vb$ assigns specific values to the variables in the subset.
    If performing an intervention results in changes in other variables of interest, it suggests a causal relationship, apart from mere correlation.
    Interventions are often represented with the $do()$ operator, highlighting the deliberate change of a variable \cite{pearl2000models, pearl2009causality}.
    For the sake of simplicity in notation, we use $\pro_{\Xb}(\cdot)$ to denote the distribution of the variables after an intervention on $\Xb$.
    Figure \ref{fig: intro 2} depicts the population after an intervention on subset $\Xb \subseteq \Vb$, where we seek to understand how changes in $\Xb$ would affect a set of outcome variables $\Yb \subseteq \Vb \setminus \Xb$.
    To analyze this causal effect across the entire population, we must compute the distribution $\pro_{\Xb}(\Yb)$.
    On the other hand, if we are merely interested in the results of the intervention on a specific sub-population, it suffices to compute the conditional distribution $\pro_{\Xb}(\Yb \vert S=1)$ pertaining to the sub-population.

    \noindent \textbf{Causal Effect Identification.}
    Performing interventions in populations can be challenging due to high costs, ethical concerns, or sheer impracticability. 
    Instead, researchers often use observational methods, leveraging the environment's \emph{causal graph}, a graphical representation that depicts the causal relationships between variables \cite{pearl2009causality, spirtes2000causation}, and observational data to estimate interventional distributions of interest.
    Various \textit{causal effect identification} problems in the causal inference literature are concerned with this issue.

    \noindent \textbf{Related Work.}
    Table \ref{table: intro} lists four causal effect identification problems.
    The most renowned among them is the ID problem, introduced by \cite{pearl1995causal}, which seeks to determine a causal effect for the entire population using the observational distribution pertaining to the entire population.
    Specifically, it aims to compute $\pro_{\Xb}(\Yb)$ from $\pro(\Vb)$.
    The c-ID problem, introduced by \cite{shpitser2012identification}, extends the ID problem to handle conditional causal effects, i.e., compute the conditional causal effect $\pro_{\Xb}(\Yb \vert \Zb)$ from the observational distribution $\pro(\Vb)$ pertaining to the entire population.
    \cite{Bareinboim_Tian_Pearl_2014} introduced the \s-Recoverability problem that focuses on inferring the causal effect of $\Xb$ on $\Yb$ for the entire population using data drawn solely from a specific sub-population.
    \cite{abouei2023sid} introduced \s-ID, which asks whether a causal effect in a sub-population such as $\pro_{\Xb}(\Yb \vert S=1)$ can be uniquely computed from the observational distribution pertaining to that sub-population, i.e., $\pro(\Vb \vert S=1)$.
    Another direction of research considers learning a causal effect from multiple datasets \cite{Lee2019GeneralIW, kivva2022revisiting, correa2021nested, kivva2023identifiability, Tikka_2021, lee2024a}.
    In all aforementioned causal inference problems, the causal graph is assumed to be known.
    Some recent work relax this assumption \cite{pmlr-v97-jaber19a, NEURIPS2022_17a9ab41} or introduce additional conditions on the causal graph with the goal of identifying a broader range of causal effects \cite{tikka2019identifying, mokhtarian2022causal, jamshidi2023causal}.
    Settings where data samples are dependent introduce new challenges to causal inference, which have been explored in another line of research \cite{sherman2018identification, bhattacharya2020causal, zhang2023causal}.

     \begin{table}[t]
        \centering
        \caption{Various causal effect identification problems. $\Xb$ is the set of intervened variables, $\Yb$ is the set of outcome variables, and $S=1$ corresponds to a sub-population.
        ID, c-ID, and \s-Recoverability have been addressed in the presence of latent variables. \s-ID problem has only been studied in causally sufficient cases where all variables are observed.}
        \label{table: intro}
        \begin{tabular}{c | c |c |c} 
            \toprule
            Problem & Given distribution & Target distribution & Presence of latent variables \\
            \hline
            ID & $\pro (\Vb)$ & $\pro_{\Xb}(\Yb)$ & \checkmark\\
            c-ID  & $\pro (\Vb)$  &  $\pro_{\Xb}(\Yb \vert \Zb)$ & \checkmark\\
            \s-Recoverability &  $\pro (\Vb \vert S = 1)$  &  $\pro_{\Xb}(\Yb)$ & \checkmark\\
            \s-ID  &  $\pro (\Vb \vert S = 1)$  &   $\pro_{\Xb}(\Yb \vert S=1)$ & \redcross \\
            \bottomrule
        \end{tabular}
    \end{table}

    \noindent  \textbf{\s-ID Is Not ID.}
    It is worth emphasizing that the \s-ID problem is not a special case of ID problem, where the population is restricted to the target sub-population. The presence of selection bias $S$ introduces additional dependencies among variables, and ignoring $S$ in the graph invalidates the application of the rules of do-calculus \cite{pearl2000models} (which are the main tools used to tackle the ID problem) on input distribution, i.e., $\pro(\Vb \vert S = 1)$.
    Consequently, there are many instances where a causal effect is identifiable in the ID setting but not identifiable in the \s-ID setting.
    In particular, when all the variables in a causal system are observable, all causal effects are identifiable in the setting of the ID problem \cite{pearl2000models}. 
    This is not the case in the \s-ID setting, and some causal effects become non-identifiable, as noted by \cite{abouei2023sid}.
    Moreover, even when a causal effect is identifiable in both the ID and \s-ID settings, using the expression from the ID algorithm can lead to erroneous inference. 
    Example \ref{ex: intro} illustrates this case.

    \noindent  \textbf{Example 1.} \label{ex: intro}
        Consider an example pertaining to study of the effect of a cholesterol-lowering medication on cardiovascular disease.
        Figure \ref{fig: intro-ex} depicts the causal graph of this example, where $X$ is the medication choice that directly affects $Y$, cardiovascular disease.
        Variable $Z$ represents the diet and exercise routine of a person.
        In this scenario, $X$ and $Z$ are confounded by the person's socioeconomic status (e.g., income).
        It can be shown that the causal effect of $X$ on $Y$ (in the entire population, for instance, the people around the globe) is identifiable (ID) from $\pro(X, Y, Z)$ and can be computed as $\pro_{X}(Y) = \pro(Y \vert X)$.

         \begin{wrapfigure}{r}{0.4\textwidth}
            \centering
            \begin{tikzpicture}
                \node [block, label=center:X](X) {};
                \node [block, right= 1 of X, label=center:Z](Z) {};
                \node [block, below = 1 of X, label=center:Y](Y) {};
                \node [block-s, right= 1 of Y, label=center:$S$](S) {};
                \draw[edge] (X) to (Y);
                \draw[edge] (Z) to (S);
                \draw[edge2] (X) to (Z);
                \draw[edge2] (Y) to (S);
            \end{tikzpicture}
            \caption{ADMG $\Grs$ in Example 1.}
            \label{fig: intro-ex}
        \end{wrapfigure}
        
        However, we might instead be interested in a study that focuses on the people of a specific region.
        In this case, the target sub-population would correspond to individuals who are biased toward particular diet and exercise routines and possibly have a higher genetic predisposition for heart disease.
        Let $S$ be an indicator node for this sub-population, $Z$ has a directed edge toward $S$, and $S$ and $Y$ are confounded by the latent genetic predisposition of the people of this group.    
        We will show in Section \ref{sec: main} that the causal effect in this sub-population, $\pro_X(Y \vert S=1)$, is \s-ID and equals $\sum_{Z} \pro(Y \vert X, Z, S=1) \pro(Z \vert S=1)$.
        In this example, the presence of $S$ introduces a spurious correlation between $X$ and $Y$ through the path involving $Z$ and $S$. Therefore, if we were to ignore the presence of $S$ and apply the ID algorithm to the input $\pro(X, Y, Z \vert S = 1)$, it would result in incorrect inference: $\pro(Y \vert X, S = 1)$ as opposed to the correct value $\pro_X(Y \vert S=1)$.
        In Appendix \ref{apd: exp}, we empirically compare the differences between ID and \s-ID settings. We present another example in Appendix \ref{apd: example} where a causal effect is ID but not \s-ID.

    In this paper, we consider the \s-ID problem in the presence of latent variables. Our main contributions are as follows.

    \begin{itemize}[leftmargin=*]
        \item We extend the classical relevant graphical definitions, such as c-components and Hedges, initially defined for the ID problem so that they inherit the key properties of their predecessors but are applicable to the \s-ID setting in the presence of latent variables (Section \ref{sec: s-ID}).
        \item We present a sufficient graphical condition to determine whether a causal effect is \s-ID (Theorem \ref{thm: main}).
        Accordingly, we propose a sound algorithm for the \s-ID problem (Algorithm \ref{algo: s-ID}).
       
        \item We show a reduction from the \s-Recoverability problem to the \s-ID problem (Theorem \ref{thm: reduction}), indicating that solving \s-ID can also solve the \s-Recoverability problem.
    
    \end{itemize}

    \noindent\textbf{Organization.} In Sections \ref{sec: preliminaries} and \ref{sec: ID}, we cover the preliminaries and review key definitions and results for the ID problem.
    In Section \ref{sec: s-ID}, we formally define the \s-ID problem in the presence of latent variables and present the proper modifications of the classical graphical notions of interest for the \s-ID problem.
    We present our main results in Section \ref{sec: main}.
    In Section \ref{sec: reduction}, we introduce a reduction from \s-Recoverability to \s-ID. 
    The appendix includes proofs of our results, as well as a numerical experiment.

\section{Preliminaries} \label{sec: preliminaries}
    Throughout the paper, we use capital letters to represent random variables and bold letters to represent sets of variables. 
    Furthermore, to facilitate ease of reading, we have summarized the key notations in Table \ref{table: notations}.

\begin{table}[t!]
    \centering
    \caption{Table of notations.}
    \label{table: notations}
    \small
    \begin{tabular}{c|c } 
        \toprule
        Notation & Description \\
        \hline
        $\Vb, \Ub$ & Sets of observed and unobserved variables \\
        $S$ &  Auxiliary vertex (variable) used to model a sub-population \\
        $\Grs$ &  Augmented ADMG over $\Vb \cup \{S\}$ \\
        $\Pa{X}{\Gr}$ &  Parents of vertex $X$ in graph $\Gr$ \\
        $\Anc{X}{\Gr}, \Anc{\Xb}{\Gr}$ & Ancestors of vertex $X$ (including $X$); the union of ancestors for all $X \in \Xb$ \\
        $\Vs, \Vns$ & $\Vb \cap \Anc{S}{\Grs}$ and its complement, $\Vb \setminus \Anc{S}{\Grs}$\\
        $\Gr[\Xb]$ & Subgraph of $\Gr$ induced by the vertices in $\Xb$\\
        $\Gr_{\overline{\Xb}\underline{\Zb}}$ & Subgraph of $\Gr$ after removing incoming edges to $\Xb$ and outgoing edges from $\Zb$  \\
        $\prs (\Vb)$ & Sub-population distribution, i.e., $\pro(\Vb \vert S = 1)$ \\
        $\pro_{\Xb}(\Yb)$ & Causal effect of $\Xb$ on $\Yb$, i.e., post-interventional distribution\\
        $\prs_{\Xb}(\Yb)$ & Causal effect of $\Xb$ on $\Yb$ in the sub-population, i.e., $\pro_{\Xb}(\Yb \vert S = 1$) \\
        $Q[\Hb]$ & $\pro_{\Vb \setminus  \Hb } \left(\Hb\right)$ \\
        $\Qs[\Hb]$ &  $\pro_{\Vns \setminus  \Hb } \left(\Hb \vert \Anc{S}{\Grs} \setminus \{S\}, S = 1 \right), \forall~ \Hb \subseteq \Vns$\\
        \bottomrule
    \end{tabular}
\end{table}

    \textbf{Graph Definitions.} Acyclic directed mixed graphs (ADMGs) consist of a mix of directed and bidirected edges and have no directed cycles.
    Let $\Gr = (\Vb, \Eb_1, \Eb_2)$ be an ADMG, where $\Vb$ is a set of variables, $\Eb_1$ is a set of directed edges ($\rightarrow$), and $\Eb_2$ is a set of bidirected edges ($\leftrightarrow$).
    The set of parents of a variable $X\in \Vb$, denoted by $\Pa{X}{\Gr}$, consists of the variables with a directed edge to $X$.
    Similarly, the set of ancestors of $X\in \Vb$, denoted by $\Anc{X}{\Gr}$, includes all variables on a \emph{directed} path to $X$, including $X$ itself.
    For a set $\Xb \subseteq \Vb$, we define $\Anc{\Xb}{\Gr}=\bigcup_{X \in \Xb}\Anc{X}{\Gr}$. In an ADMG $\Gr$ over $\Vb$, a subset $\Xb \subseteq \Vb$ is called ancestral if $\Anc{\Xb}{\Gr} = \Xb$.

    A path is called bidirected if it only consists of bidirected edges.
    A non-endpoint vertex $X_i$ on a path $(X_1, X_2, \dots, X_k)$ is called a \emph{collider} if one of the following situations arises:
    \begin{align*}
        X_{i-1} \to X_i \gets X_{i+1},\quad X_{i-1} \leftrightarrow X_i \gets X_{i+1}, \quad
        X_{i-1} \to X_i \leftrightarrow X_{i+1}, \quad
        X_{i-1} \leftrightarrow X_i \leftrightarrow X_{i+1}.
    \end{align*}

    Let $\Xb,\Yb,\Wb$ be three disjoint subsets of variables in an ADMG $\Gr$.
    A path $\mathcal{P}= (X, Z_1, \dots, Z_k, Y)$ between $X \in \Xb$ and $Y \in \Yb$ in $\Gr$ is called \emph{blocked} by $\Wb$ if there exists $1 \leq i \leq k$ such that $Z_i$ is a collider on $\mathcal{P}$ and $Z_i \notin \Anc{\Wb}{\Gr}$, or $Z_i$ is not a collider on $\mathcal{P}$ and $Z_i \in \Wb$.
    
    Denoted by $(\Xb \independent \Yb \vert \Wb)_{\Gr}$, we say $\Wb$ $m$-separates $\Xb$ and $\Yb$ if for any $X \in \Xb$ and $Y \in \Yb$, $\Wb$ blocks all the paths in $\Gr$ between $X$ and $Y$.
    Conversely,  $(\Xb \notindependent \Yb \vert \Wb)_{\Gr}$ if there exists at least one path between a variable in $\Xb$ and a variable in $\Yb$ that is not blocked by $\Wb$.
    
    For $\Xb, \Zb \subseteq \Vb$, $\Gr_{\overline{\Xb}\underline{\Zb}}$ denotes the edge subgraph of $\Gr$ obtained by removing the edges with an arrowhead to a variable in $\Xb$ (including bidirected edges) and outgoing edges of $\Zb$ (excluding bidirected edges). Moreover, $\Gr[\Xb]$ denotes the vertex subgraph of $\Gr$ consisting of $\Xb$ and bidirected and directed edges between them.

    \textbf{SCM.} A structural causal model (SCM) is a tuple $\left(\Vb, \Ub, \Fb, P(\Ub) \right)$, where $\Vb$ is a set of endogenous variables, $\Ub$ is a set of exogenous variables independent from each other with the joint probability distribution $P(\Ub)$, and $\Fb = \{f_X\}_{X \in \Vb}$ is a set of deterministic functions such that for each $X \in \Vb$,
    \begin{equation*}
        X = f_X(\textit{Pa}^X, \Ub^X),
    \end{equation*}
    where $\textit{Pa}^X \subseteq \Vb \setminus \{X\}$ and $\Ub^X \subseteq \Ub$.
    This SCM induces a causal graph $\Gr$ over $\Vb$ such that $\Pa{X}{\Gr} = \textit{Pa}^X$ and there is a bidirected edge between two distinct variables $X, Y \in \Vb$ when $\Ub^X \cap \Ub^Y \neq \varnothing$.
    Henceforth, we assume the underlying SCM induces a causal graph that is ADMG, i.e., it contains no directed cycles.

    An SCM $\M= \left(\Vb, \Ub, \Fb, P(\Ub) \right)$ with causal graph $\Gr$ induces a unique joint distribution over the variables $\Vb$ that can be factorized as
    \begin{equation*}
        \pro^{\M}(\Vb) 
        = \sum_{\Ub} \prod_{X\in \Vb} \pro^{\M}(X \vert \Pa{X}{\Gr}) \prod_{U \in \Ub} \pro^{\M}(U).
    \end{equation*}
    This property is known as the Markov factorization \cite{pearl2009causality}.
    Note that $\sum_{\Xb}$ denotes marginalization, i.e., summation (or integration for continuous variables) over all the realizations of the variables in set $\Xb$.
    We often drop the $\M$ in $\pro^{\M}(\cdot)$ when it is clear from the context.
    
    \textbf{Modeling a Sub-Population.} We model a sub-population using an auxiliary variable $S$ and a biased sampler from a causal environment akin to \cite{pmlr-v22-bareinboim12, abouei2023sid}. Suppose $\M$ is the underlying SCM of an environment with the set of observed variables $\Vb$.
    In this causal environment, an unbiased sampler produces samples drawn from $\pro(\Vb)$.
    When the sampler is biased, it draws samples from the conditional distribution $\prs(\Vb) \coloneqq \pro(\Vb \vert S=1)$, where $S$ is an auxiliary variable defined as
        $ S \coloneqq f_S(\text{Pa}^{S}, \Ub^{S}),$
    where $f_S$ is a binary function, $\text{Pa}^S \subseteq \Vb$, and $\Ub^S$ is the set of exogenous variables corresponding to $S$.
    Note that $\Ub^S$ can intersect with $\Ub$, but the variables in $\Ub \cup \Ub^S$ are assumed to be independent.
    In this model, $S=1$ indicates that the sample is drawn from the target sub-population.
    Furthermore, we define the augmented SCM $\M^{\s}= \left(\Vb \cup \{S\}, \Ub \cup \Ub^S, \Fb \cup \{f_S\}, P(\Ub \cup \Ub^S) \right)$ obtained by adding $S$ to the underlying SCM $\M$.
    We denote by $\Grs$, the causal graph of $\M^{\s}$, which is an augmented ADMG over $\Vb \cup \{S\}$.
    Note that in $\Grs$, variable $S$ does not have any children, but it can have several parents and bidirected edges.

    \textbf{Intervention.} An \emph{intervention} on a set $\Xb \subseteq \Vb$ converts $\M$ to a new SCM where the equations of the variables in $\Xb$ are replaced by some constants.
    We denote by $Q[\Vb \setminus \Xb] \eqdef \pro_{\Xb}(\Vb \setminus \Xb)$ the corresponding post-interventional distribution, i.e., the joint distribution of the variables in the new SCM. The causal effect of $\Xb$ on $\Yb$ refers to the post-interventional distribution $\pro_{\Xb}(\Yb)$, where $\Xb$ and $\Yb$ are disjoint subsets of $\Vb$. Accordingly, the causal effect of $\Xb$ on $\Yb$ in a sub-population is denoted by $\pro_{\Xb}(\Yb \vert S = 1)$.\footnote{Note that in this notation, the order of operations is intervention first, then conditioning.}

    \textbf{Problem Setup.} Let $(\Vb, \Ub, \Fb, P(\Ub))$ be an SCM with ADMG $\Gr$ representing its causal graph. 
    Additionally, let $S$ be an auxiliary variable representing a specific sub-population. 
    In this paper, given the augmented graph $\Grs$ and two arbitrary, disjoint subsets $\Xb$ and $\Yb$, we address the following question: Can the causal effect $\prs_{\Xb}(\Yb)$ be uniquely identified from the observational distribution $\prs(\Vb)$? Please refer to Definition \ref{def:S-ID} for the formal definition of the \s-ID problem.
        

\section{ID, \cco, and Hedge} \label{sec: ID}
    Our proposed approach to address the \s-ID problem extends certain definitions and properties from the classic ID problem \cite{pearl1995causal}.
    For the sake of completeness and pedagogical reasons, in this section, we review some definitions and the main results in the ID problem \cite{tian2002general, huang2006identifiability, shpitser2006identification}.

    \begin{definition}[ID]
        Suppose $\Gr$ is an ADMG over $\Vb$ and let $\Xb$ and $\Yb$ be disjoint subsets of $\Vb$.
        Causal effect $\pro_{\Xb}(\Yb)$ is said to be identifiable (or ID for short) in $\Gr$ if for any two SCMs $\M_1$ and $\M_2$ with causal graph $\Gr$ for which $\pro^{\M_1}(\Vb) = \pro^{\M_2}(\Vb) > 0$,  then $\pro^{\M_1}_{\Xb}(\Yb) = \pro^{\M_2}_{\Xb}(\Yb)$.
    \end{definition}

    Next, we review \ccos, a fundamental concept to address the ID problem.

    \begin{definition}[\cco] \label{def: c-component}
        Suppose $\Gr$ is an ADMG over $\Vb$.
        The \ccos of $\Gr$ are the connected components in the graph obtained by considering only the bidirected edges of $\Gr$.
        Furthermore, $\Gr$ is called a single \cco if it contains only one \cco.
    \end{definition}

    There exist a few different definitions for \emph{Hedge}, another central notion in the ID literature.
    Here, we provide a somewhat simplified definition that not only suffices to present the main result of the ID problem but also allows us to extend it in the next section to the \s-ID setting.
    
    \begin{definition}[Hedge] \label{def: hedge}
        Suppose $\Gr$ is an ADMG over $\Vb$, and let $\Yb \subseteq \Vb$ such that $\Gr[\Yb]$ is a single \cco.
        A subset $\Hb \subseteq \Vb$ is called a Hedge for $\Yb$ in $\Gr$, if $\Yb \subsetneq \Hb$, $\Gr[\Hb]$ is a single \cco, and $\Hb = \Anc{\Yb}{\Gr[\Hb]}$.
    \end{definition}

    \setcounter{example}{1}
    \begin{example} \label{ex: c-component and Hedge}
        Consider the ADMG $\Gr$ over $\Vb = \{X_1, X_2, Y_1, Y_2\}$ depicted in Figure \ref{fig: ID}.
        In this case, $\Gr$, $\Gr[X_1, X_2]$, and $\Gr[Y_1, Y_2]$ are single \ccos.
        The \ccos of $\Gr[X_1, X_2, Y_1]$ are $\{X_1, X_2\}$ and $\{Y_1\}$.
        Furthermore, $\Hb = \{X_1, Y_1, Y_2\}$ is a Hedge for $\Yb = \{Y_1, Y_2\}$ since $\Gr[\Hb]$ and $\Gr[\Yb]$ are single \ccos and $\Anc{\Yb}{\Gr[\Hb]} = \Hb$.
        Similarly, $\Vb$ is a Hedge for $\Yb$, but there exists no Hedge for either $\{Y_1\}$ or $\{Y_2\}$.
    \end{example}  

    The following theorem, restating the results in \cite{shpitser2006identification} and \cite{huang2006identifiability}, outlines a necessary and sufficient condition to determine the identifiability of a causal effect in an ADMG.
    
    \begin{theorem}[ID] \label{thm: ID}
        Let $\Gr$ be an ADMG over $\Vb$, and $\Xb$ and $\Yb$ be two disjoint subsets of $\Vb$.
        Causal effect $\pro_{\Xb}(\Yb)$ is ID in $\Gr$ if and only if $Q[\Db]$ is ID in $\Gr$, where $\Db = \Anc{\Yb}{\Gr[\Vb \setminus \Xb]}$.
        Furthermore, let $\{\Db_i\}_{i=1}^k$ be the \ccos of $\Gr[\Db]$, then $Q[\Db]$ is ID in $\Gr$ if and only if there are no Hedge in $\Gr$ for any of the \ccos $\{\Db_i\}_{i=1}^k$.
    \end{theorem}

    \begin{example} \label{ex: ID}
        Following Example \ref{ex: c-component and Hedge}, Theorem \ref{thm: ID} implies that $P_{X_1}(Y_1)$, $P_{X_2}(Y_2)$, $P_{\{X_1, X_2\}}(Y_1)$, and $P_{\{X_1, X_2\}}(Y_2)$ are ID since no Hedge for either $\{Y_1\}$ or $\{Y_2\}$ exists.
        However, $P_{X_1}(Y_1, Y_2)$ is not ID because $\Db = \{Y_1, Y_2, X_2\}$ and the \ccos of $Q[\Db]$ are $\Db_1 = \{Y_1, Y_2\}$ and $\Db_2 = \{X_2\}$, and $\{X_1, Y_1, Y_2\}$ (or $\Vb$) is a Hedge for $\Db_1$.
        Similarly, we can show that $P_{\{X_1, X_2\}}(Y_1, Y_2)$ is not ID.
    \end{example}


\begin{figure}
    \begin{subfigure}[b]{0.45\textwidth}  
        \centering
        \begin{minipage}[c][4cm][c]{\textwidth}
            \centering
            \begin{tikzpicture}
                \node [block, label=center:$X_1$](X1) {};
                \node [block, right= 1 of X1, label=center:$X_2$](X2) {};
                \node [block, below = 1 of X1, label=center:$Y_1$](Y1) {};
                \node [block, right= 1 of Y1, label=center:$Y_2$](Y2) {};

                \draw[edge] (X1) to (Y1);
                \draw[edge] (X2) to (Y2);
                \draw[edge2] (X1) to (X2);
                \draw[edge2] (Y1) to (Y2);
                \draw[edge2] (X1) to (Y2);
            \end{tikzpicture}
        \end{minipage}
        \caption{ADMG $\Gr$ in Examples 2-3.}
        \label{fig: ID}
    \end{subfigure}
    \begin{subfigure}[b]{0.45\textwidth}
        \centering
        \begin{minipage}[c][4cm][c]{\textwidth}
            \centering
            \begin{tikzpicture}
                \node [block, label=center:$X_1$](X) {};
                \node [block, below = 0.5 of X, label=center:$X_2$](X2) {};
                \node [block, right= 0.75 of X2, label=center:$Z_2$](Z2) {};
                \node [block, right= 0.75 of X, label=center:$Z_1$](Z1) {};
                \node [block, below = 0.5 of X2, label=center:$Y_2$](Y2) {};
                \node [block-s, right = 0.75 of Y2, label=center:$S$](S) {};
                \node [block, left = 0.5 of X2, label=center:$Y_1$](Y1) {};   

                \draw[edge] (X) to (Y1);
                \draw[edge] (Y1) to (Y2);
                \draw[edge] (X) to (X2);
                \draw[edge] (X2) to (Y2);
                \draw[edge] (Z2) to  (Y2);
                \draw[edge] (Z1) to (Z2);
                \draw[edge] (Z2) to (S);
                \draw[edge] (Z2) to (Y2);
                \draw[edge2] (Z1) to  (X);
                \draw[edge2] (Z2) to (X);
                \draw[edge2] (Z2) to  (X2);
                \draw[edge2] (Y2) to  (S);
                \draw[edge2] (Y1) to [bend right = 70] (S);
            \end{tikzpicture}
        \end{minipage}
        \caption{Augmented ADMG $\Grs$ in Examples 4-9.}
        \label{fig: ex-s-ID}
    \end{subfigure}
    \caption{ADMGs in Examples of Sections \ref{sec: ID}, \ref{sec: s-ID}, and \ref{sec: main}.}
\end{figure}

\section{\s-ID, \sco, and \s-Hedge} \label{sec: s-ID}
    We begin by providing a formal definition of the \s-ID problem in the presence of latent variables, i.e., when the causal graph is an ADMG.
    Then, we present modifications of the graphical notions from the previous section so that they inherit the key properties of their predecessors and can be applied to the \s-ID setting. 

    To avoid repetition, henceforth, we denote by $\Vb$ the set of observed variables and by $\Grs$ an augmented ADMG over $\Vb \cup \{S\}$.
    Furthermore, we denote by $\Vs$ and $\Vns$ the ancestors and non-ancestors of $S$ in $\Vb$, i.e.,
    \begin{equation*}
        \Vs \coloneqq \Vb \cap \Anc{S}{\Grs}, \quad \Vns \coloneqq \Vb \setminus \Anc{S}{\Grs}.
    \end{equation*}
    
    \begin{definition}[\s-ID]\label{def:S-ID}
        Let $\Xb$ and $\Yb$ be disjoint subsets of $\Vb$.
        Conditional causal effect $\pro_{\Xb}(\Yb \vert S=1)$ (or $\prs_{\Xb}(\Yb)$) is  \s-ID in $\Grs$ if for any two augmented SCMs $\M_1^{\s}$ and $\M_2^{\s}$ with causal graph $\Grs$ for which $\pro^{\M_1^{\s}}(\Vb \vert S=1) = \pro^{\M_2^{\s}}(\Vb \vert S=1) > 0$, then $\pro^{\M_1^{\s}}_{\Xb}(\Yb \vert S=1) = \pro^{\M_2^{\s}}_{\Xb}(\Yb \vert S=1)$.
    \end{definition}

    Next definition extends $Q[\cdot]$ and introduces $\Qs{[\cdot]}$.
    \begin{definition}[$\Qs{[\cdot]}$]
        For $\Hb \subseteq \Vns$, we define
           $ \Qs[\Hb] \coloneqq 
            \pro_{\Vns \setminus  \Hb } \left(\Hb \vert \Anc{S}{\Grs} \setminus \{S\}, S = 1 \right).$
    \end{definition}

    The next definition extends \ccos (Definition \ref{def: c-component}) and introduces \scos.

    \begin{definition}[\sco] \label{def: s-comp}
        For a subset $\Hb \subseteq \Vns $, let $\Cb_1, \dots, \Cb_k$ denote the \ccos of $\Grs[\Hb \cup \Anc{S}{\Grs}]$.
        We define the \scos of $\Hb$ in $\Grs$ as the subsets $\Hb_i \coloneqq \Cb_i \cap \Hb$ which are non-empty.
        Furthermore, $\Hb$ is called a single \sco in $\Grs$ if it contains only one \sco.
    \end{definition}
    Note that $\Qs[\cdot]$ and \scos are only defined for the subsets of $\Vns$.
    Figure \ref{fig: scomps} visualizes the structure of \s-components of a subset $\Hb \subseteq \Vns$.
    In this figure, each blue subset (e.g., $M_1$) represents a c-component, which means all the nodes within them are connected via bidirected edges.
    Therefore, according to Definition \ref{def: s-comp}, all nodes inside \s-components (e.g., $\Hb_1$) of $\Hb$ are connected via bidirected edges in $\Grs[\Hb \cup \Vs]$. Figure \ref{fig: single-s} shows the structure of a single \s-component, where all the nodes of $\Hb$ are connected via bidirected edges in $\Grs[\Hb \cup \Vs]$.

    \begin{figure}
        \centering
        \begin{subfigure}[t]{.22\textwidth} 
            \centering
            \begin{minipage}[b][4cm][c]{\textwidth} 
                \centering
                \includegraphics[width=\textwidth]{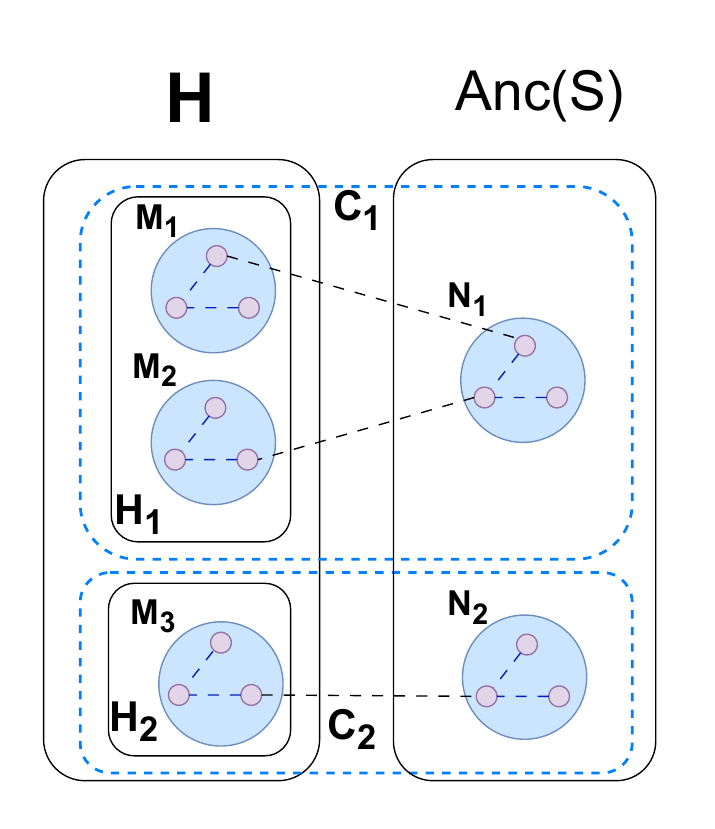}
            \end{minipage}
            \caption{\s-components of $\Hb$ \\ are $\Hb_1, \Hb_2$.}
            \label{fig: scomps}
        \end{subfigure}
        \hfill
        \begin{subfigure}[t]{.22\textwidth}  
            \centering
            \begin{minipage}[b][4cm][c]{\textwidth}
                \centering
                \includegraphics[width=\textwidth]{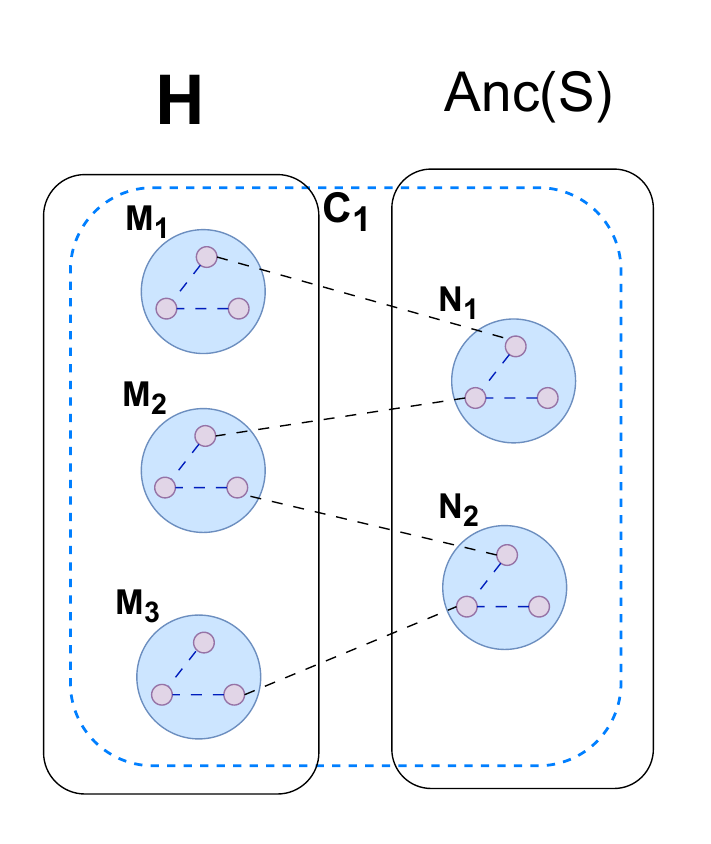}
            \end{minipage}
            \caption{$\Hb$ is a single \\ \s-component.}
            \label{fig: single-s}
        \end{subfigure}
        \hfill
        \begin{subfigure}[t]{.32\textwidth}  
            \centering
            \begin{minipage}[b][4cm][c]{\textwidth}
                \centering
                \includegraphics[width=\textwidth]{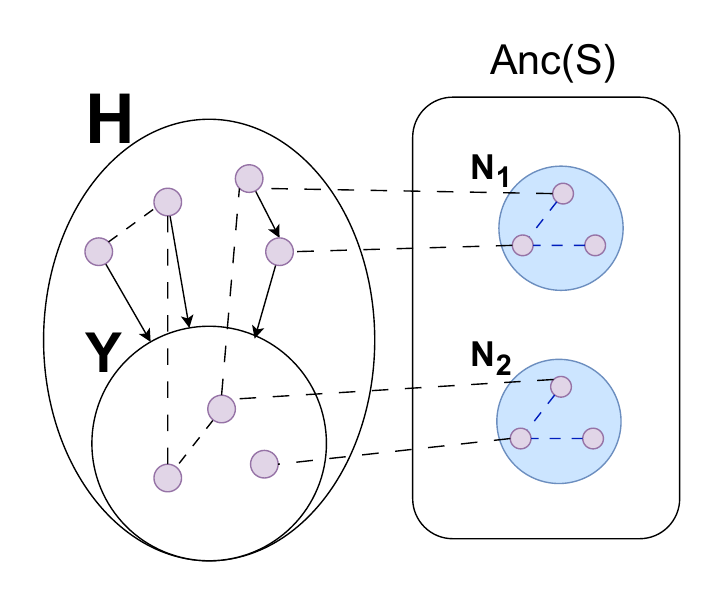}
            \end{minipage}
            \caption{$\Hb$ is an \s-Hedge for $\Yb$.}
            \label{fig: s-hedge}
        \end{subfigure} 
        \hfill
        \begin{subfigure}[t]{.21\textwidth}  
            \centering
            \begin{minipage}[b][4cm][c]{\textwidth}
                \centering
                \includegraphics[width=0.85\textwidth]{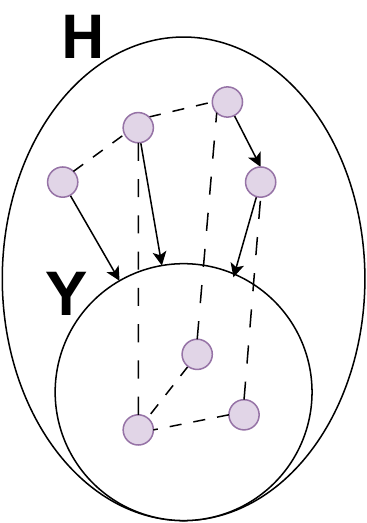}
            \end{minipage}
            \caption{$\Hb$ is a Hedge for $\Yb$.}
            \label{fig: hedge}
        \end{subfigure}
        \caption{Visualization of the graph structures defined in Sections \ref{sec: ID} and \ref{sec: s-ID}.}
        \label{fig:enter-label}
    \end{figure}

    \begin{example} \label{ex: ex4}
        Consider the ADMG $\Grs$ in Figure \ref{fig: ex-s-ID} over $\Vb \cup \{S\}$, where $\Vb = \{X_1, X_2, Y_1, Y_2, Z_1, Z_2\}$.
        Since $\Anc{S}{\Grs} = \{Z_1, Z_2, S\}$, we have $\Vs = \{Z_1, Z_2\}$ and $\Vns = \{X_1, X_2, Y_1, Y_2\}$.
        In this case, the \scos of $\Vns$ are $\{X_1, X_2\}$ and $\{Y_1, Y_2\}$.
        Moreover, the \scos of $\{X_1, Y_1, Y_2\}$ are $\{X_1\}$ and $\{Y_1, Y_2\}$.
    \end{example}    
 
    We now provide two crucial properties for $\Qs[\cdot]$.
    
    \begin{lemma} \label{lemma: qs-margin}    
        Let $\Wb, \Wb'$ be two subsets of $\Vns$ such that $\Wb' \subset \Wb$.
        If $\Wb'$ is an ancestral set in $\Grs[\Wb]$, then
        \begin{equation} \label{eq: qs-margin}
            \Qs[\Wb'] = \sum_{\Wb \setminus \Wb'} \Qs[\Wb].
        \end{equation}
    \end{lemma}

    \begin{lemma} \label{lemma: qs-decom}
        Suppose $\Hb \subseteq \Vns$ and let $\Hb_1, \dots, \Hb_k$ denote the \scos of $\Hb$ in $\Grs$.
        Then,
        \begin{itemize}[leftmargin=*]
            \item $\Qs[\Hb]$ decomposes as
                \begin{equation}
                    \Qs[\Hb] = \Qs[\Hb_1]\Qs[\Hb_2]\dots \Qs[\Hb_k].
                \end{equation}
            \item Let $m$ be the number of variables in $\Hb$, and consider a topological ordering of  the variables in graph $\Grs[\Hb]$, denoted as $V_{h_1} < \dots < V_{h_m}$.
            Let $\Hb^{(0)} = \emptyset$ and for each $1 \leq i \leq  m$, $\Hb^{(i)}$ denote the set of variables in $\Hb$ ordered before $V_{h_i}$ (including $V_{h_i}$).
            For every $1\leq j \leq k$, $\Qs[\Hb_{j}]$ can be computed from $\Qs[\Hb]$ by   
            \begin{equation}
                \label{eq: qs_lemma_prod}    
                \Qs[\Hb_{j}] = \prod_{\{i \vert V_{h_{i}} \in \Hb_{j} \}} \frac{\Qs[\Hb^{(i)}]}{\Qs[\Hb^{(i-1)}]},
            \end{equation}
            where $\Qs[\Hb^{(i)}]$s can be computed by
            \begin{equation}    
                 \Qs[\Hb^{(i)}] = \mysum{\Hb \setminus \Hb^{(i)}}{} \Qs[\Hb].
            \end{equation}
        \end{itemize} 
    \end{lemma}
    
    The aforementioned lemmas are extensions of similar lemmas for $Q[\cdot]$ \cite{tian2003ID} to  $\Qs[\cdot]$.

    \begin{example} \label{ex: ex5}
        Following Example \ref{ex: ex4}, since $\{Y_1\}$ is ancestral in $\Grs[Y_1, Y_2]$, Lemma \ref{lemma: qs-margin} implies that $\Qs[Y_1] = \sum_{Y_2}\Qs[Y_1, Y_2]$.
        Furthermore, since the \scos of $\Vns$ are $\{X_1, X_2\}$ and $\{Y_1, Y_2\}$, Lemma \ref{lemma: qs-decom} implies that $\Qs[Y_1, Y_2] = \frac{\Qs[\Vns]}{\sum_{Y_1, Y_2} \Qs[\Vns]}$. Thus, $\Qs[Y_1]$ can be computed from $\Qs[\Vns]$.
    \end{example}
    
    Finally, we define \s-Hedges, which extends Definition \ref{def: hedge} for Hedges.

    \begin{definition}[\s-Hedge] \label{def: s-Hedge}
        Suppose $\Yb \subseteq \Vns$ is a single \sco in $\Grs$.
        A subset $\Hb \subseteq \Vns$ is called an \s-Hedge for $\Yb$ in $\Grs$, if $\Yb \subsetneq \Hb$, $\Hb$ is a single \sco in $\Grs$, and $\Hb = \Anc{\Yb}{\Grs[\Hb]}$.
    \end{definition}

     When $\Hb \subseteq \Vns$ is a single c-component, it is also a single \s-component.
     Therefore, if $\Hb$ is a Hedge for $\Yb$, it will also be an \s-Hedge for $\Yb$.
     Thus, Hedges can be seen as special cases of \s-Hedges when $\Hb \subseteq \Vns$.
     Figure \ref{fig: hedge} shows the structure of $\Hb$, which is a single c-component and forms a Hedge for $\Yb$.
     Moreover, Figure \ref{fig: s-hedge} presents the structure of an \s-hedge $\Hb$ for $\Yb$.
     Note that \s-hedges are more complex graph structures compared to Hedges.
     This complexity is required for us to be able to determine whether a causal effect is \s-ID.

    \begin{example} \label{ex: ex6}
        Following Examples \ref{ex: ex4} and \ref{ex: ex5}, $\{X_1, X_2\}$ is an \s-Hedge for $\{X_2\}$, because both $\{X_1, X_2\}$ and $\{X_2\}$ are single \scos and $\{X_1, X_2\} = \Anc{X_2}{\Grs[X_1, X_2]}$.
        Similarly, $\{Y_1, Y_2\}$ is an \s-Hedge for $\{Y_2\}$.
    \end{example}


\section{Main Results} \label{sec: main}
    In this section, we provide a sufficient graphical condition for a causal effect to be \s-ID in an ADMG.
    This extends the condition presented in \cite{abouei2023sid}, which assumes that the causal graph is a DAG.
    Accordingly, we propose a sound algorithm for the \s-ID problem in the presence of latent variables.

    Recall that $\Grs$ is an augmented ADMG over the set observed variables $\Vb$ and auxiliary variable $S$, and we defined $\Vs = \Vb \cap \Anc{S}{\Grs}$ and $\Vns = \Vb \setminus \Anc{S}{\Grs}$.

    \begin{theorem} \label{thm: main}
        For disjoint subsets $\Xb$ and $\Yb$ of $\Vb$, let $\Xs \coloneqq \Xb \cap \Vs$, $\Xns \coloneqq \Xb \cap \Vns$, and $\Yns \coloneqq \Yb \cap \Vns$.
        \begin{enumerate}[leftmargin =*]
            \item Conditional causal effect $\prs_{\Xb}(\Yb)$ is \s-ID in $\Grs$ if and only if
            \begin{equation} \label{eq: main thm}
                (\Xs \independent \Yb \vert \Xns, S)_{\Grs_{\underline{\Xs}\overline{\Xns}}},
            \end{equation}
            and $\prs_{\Xns}(\Yb, \Xs)$ is \s-ID in $\Grs$.
            \item Suppose $\Db \coloneqq \Anc{\Yns}{\Grs[\Vns \setminus \Xns]}$ and let $\{\Db_i\}_{i=1}^k$ denote the \scos of $\Db$ in $\Grs$.
            Conditional causal effect $\prs_{\Xns}(\Yb, \Xs)$ is \s-ID in $\Grs$ if there are no \s-Hedge in $\Grs$ for any of $\{\Db_i\}_{i=1}^k$ .
        \end{enumerate}
    \end{theorem}

    \begin{remark}
        If either $\Xns$ or $\Yns$ is an empty set, then $\prs_{\Xb}(\Yb)$ is \s-ID in $\Grs$ if and only if Equation \eqref{eq: main thm} holds.
    \end{remark}

    In the absence of latent variables, i.e., when $\Grs$ is a directed acyclic graph (DAG), there are no \s-Hedge in $\Grs$ since all the edges are directed.
    Therefore, Theorem \ref{thm: main} states that $\prs_{\Xns}(\Yb, \Xs)$ is always \s-ID, and $\prs_{\Xb}(\Yb)$ is \s-ID in $\Grs$ if and only if Equation \eqref{eq: main thm} holds.
    We note that this is consistent with the condition presented in \cite[Theorem 2]{abouei2023sid} for the \s-ID problem in the absence of latent variables.

   \begin{example} \label{ex: ex7}
        Consider again ADMG $\Grs$ in Figure \ref{fig: ex-s-ID}, where we want to determine whether $\prs_{\{X_1, X_2, Z_1\}}(Y_1, Y_2)$ is \s-ID in $\Grs$.
        In this case, $\Xs = \{Z_1\}$, $\Xns = \{X_1, X_2\}$, $\Yns = \{Y_1, Y_2\}$, and
        $
            (Z_1 \independent \{Y_1, Y_2\} \vert X_1, X_2, S)_{\Grs_{\underline{Z_1}\overline{X_1, X_2}}}.
        $
        This shows that Equation \eqref{eq: main thm} holds.
        Hence, we need to determine whether $\prs_{X_1, X_2}(Y_1, Y_2, Z_1)$ is \s-ID in $\Grs$.
        In this case, $\Db = \Anc{Y_1, Y_2}{\Grs[Y_1, Y_2]} = \{Y_1, Y_2\}$, which is a single \sco.
        Since there exists no \s-Hedge for $\{Y_1, Y_2\}$, Theorem \ref{thm: main} implies that $\prs_{\{X_1, X_2, Z_1\}}(Y_1, Y_2)$ is \s-ID in $\Grs$.
   \end{example}

    \noindent \textbf{Algorithm for \s-ID.} \label{subsec: alg-sid}
    So far, we have presented a graphical condition to determine whether a causal effect is \s-ID.
    In this section, we propose a recursive algorithm that returns an expression for $\prs_{\Xb}(\Yb)$ in terms of $\prs(\Vb)$ when the condition of Theorem \ref{thm: main} holds, and otherwise, returns \textsc{Fail}.

    In the proof of Theorem \ref{thm: main} presented in the appendix, we show that when Equation \eqref{eq: main thm} holds, then
    \begin{equation} \label{eq: expression 1}
         \prs_{\Xb}(\Yb)
         = \sum_{\Wb} \prs(\Ys, \Wb \vert \Xs) \prs_{\Xns}(\Yns \vert \Vs),
    \end{equation}
    where $\Wb = \Vs \setminus (\Xs \cup \Ys)$.
    Thus, it suffices to find an expression for $\prs_{\Xns}(\Yns \vert \Vs)$ in terms of $\prs(\Vb)$.
    Let $\Db = \Anc{\Yns}{\Grs[\Vns \setminus \Xns]}$ and $\{\Db_{i}\}_{i = 1}^{k}$ be the \scos of $\Db$ in $\Grs$.
    From Lemmas \ref{lemma: qs-margin} and \ref{lemma: qs-decom} we have
    \begin{equation} \label{eq: expression 2}
        \prs_{\Xns}(\Yns \vert \Vs) = \sum_{\Db \setminus \Yns} \prod\limits_{i}\Qs[\Db_i].
    \end{equation}
    Therefore, it suffices to find an expression for each $\Db_i$ in terms of $\prs(\Vb)$.
    Note that $\Db_i$ is a single \sco in $\Grs$.
    We can now propose Algorithm \ref{algo: s-ID} for computing $\prs_{\Xb}(\Yb)$ from $\prs(\Vb)$.

    \begin{algorithm}[t]
         \caption{Computing $\interSbp$ from $\prs(\Vb)$}
         \label{algo: s-ID}
        \begin{minipage}[t]{0.5\textwidth}
            \begin{algorithmic}[1]
                \STATE \textbf{Function sID}($\Xb, \Yb, \Grs, \prs(\Vb)$)
                \STATE \textbf{Output:} Expression for $\prs_{\Xb}(\Yb)$ in terms of $\prs$ or \textsc{Fail}
        
                \STATE $\Vs \gets \Vb \cap \Anc{S}{\Grs},  \Vns \gets \Vb \setminus \Anc{S}{\Grs}$
                \STATE $\Xs \gets \Xb \cap \Vs, \quad \Xns \gets \Xb \cap \Vns$
                \STATE $\Ys \gets \Yb \cap \Vs, \quad \Yns \gets \Yb \cap \Vns$
                \IF{$(\Xs \notindependent \Yb \vert \Xns, S)_{\Grs_{\underline{\Xs}\overline{\Xns}}}$}
                    \STATE \textbf{Return} \textsc{Fail}
                \ENDIF
                \STATE $\mathbf{\Db} \gets \Anc{\Yns}{\Grs[\Vns \setminus \Xns]}$
                \STATE $\{\Db_1,\dots,\Db_k\} \gets$ \scos of $\Db$ in $\Grs$ 
                \FOR{$i$ in $[1:k]$} 
                    \STATE $\Tb_i \gets$ The \sco of $\Vns$ that contains $\Db_i$
                    \STATE Compute $\Qs[\Tb_{i}]$ using Lemma \ref{lemma: qs-decom}
                    \STATE $\Qs[\Db_i] \gets$ \textbf{sID-Single}$(\Db_i, \Tb_i, \Qs[\Tb_i])$
                    \IF{$\Qs[\Db_i]$ = \textsc{Fail}}
                        \STATE \textbf{Return} \textsc{Fail}   
                    \ENDIF
                \ENDFOR
                \STATE $\Wb \gets \Vs \setminus (\Xs \cup \Ys)$ 
                \STATE \textbf{Return} $\mysum{\Wb}{} \prs(\Ys, \Wb \vert \Xs)\mysum{\Db \setminus \Yns}{}  \prod\limits_{i} \Qs[\Db_i]$
            \end{algorithmic}  
        \end{minipage}
        \vrule 
        \hfill
        \begin{minipage}[t]{0.5\textwidth}
            \small
            \begin{algorithmic}[1]
                \STATE \textbf{Function sID-Single}($\Cb, \Tb, \Qs[\Tb]$)
                \STATE \textbf{Input}: Two single \scos $\Cb$ and $\Tb$ in $\Grs$ such that $\Cb \subseteq \Tb$
                \STATE \textbf{Output:} Expression for $\Qs[\Cb]$ in terms of $\Qs[\Tb]$ or \textsc{Fail}
                \STATE $\Ab \gets \Anc{\Cb}{\Grs[\Tb]}$
                \STATE \textbf{if} {$\Ab = \Cb$}: \textbf{Return} $\sum_{\Tb \setminus \Cb} \Qs[\Tb]$
                \STATE \textbf{if} {$\Ab = \Tb$}: \textbf{Return} \textsc{Fail}
                \IF{$\Cb \subsetneq \Ab \subsetneq \Tb,$}
                    \STATE $\Tb'$ $\gets$ The \sco of $\Ab$ in $\Grs$ that contains $\Cb$
                    \STATE Compute $\Qs[\Tb']$ from $\Qs[\Tb]$ using Lemma \ref{lemma: qs-decom}
                    \STATE \textbf{Return} \textbf{sID-Single}($\Cb$, $\Tb'$, $\Qs[\Tb']$)
                \ENDIF
            \end{algorithmic}
        \end{minipage}
    \end{algorithm}
    
    Function \textbf{sID} takes disjoint subsets $\Xb$ and $\Yb$ of $\Vb$ along with an augmented ADMG $\Grs$ and conditional distribution $\prs(\Vb)$ as input.
    After defining the required notations in lines 3-5, it checks Equation \eqref{eq: main thm} in line 6.
    If this condition is met, it defines $\Db$ and its \scos $\{\Db_i\}_{i=1}^k$ in $\Grs$.
    For each $1 \leq i \leq k$, it finds the corresponding \sco $\Tb_i$ of $\Vns$ in $\Grs$ that contains $\Db_i$.
    Note that $\Tb_i$ is well-defined as $\Db_i$ cannot partially intersect with the \scos of $\Vns$ in $\Grs$.
    Next, the algorithm seeks to compute $\Qs[\Db_i]$ from $\Qs[\Tb_i]$ by calling Function \textbf{sID-Single}.
    If Function \textbf{sID-Single} succeeds in returning an expression for each $i$, then the algorithm uses Equations \eqref{eq: expression 1} and \eqref{eq: expression 2} to return an expression for $\prs_{\Xb}(\Yb)$ in terms of $\prs(\Yb)$.
    Otherwise, the algorithm returns \textsc{Fail}.

    Function \textbf{sID-Single} takes two single \scos $\Cb$ and $\Tb$ in $\Grs$ such that $\Cb \subseteq \Tb$ and aims to drive an expression for $\Qs[\Cb]$ in terms of $\Qs[\Tb]$.
    The procedure is recursive and uses Lemmas \ref{lemma: qs-margin} and \ref{lemma: qs-decom}.
    In each recursion, the algorithm reduces $\Tb$ to a smaller subset $\Tb'$ such that $\Tb'$ is still a single \sco in $\Grs$ and $\Cb \subseteq \Tb'$ (lines 7-10).
    Eventually, the function either returns \textsc{Fail} or an expression for $\Qs[\Cb]$.

    \begin{example} \label{ex: ex8}
        Consider again ADMG $\Grs$ depicted in Figure \ref{fig: ex-s-ID}, where we want to apply Algorithm \ref{algo: s-ID} for causal effect $\prs_{X_2}(Y_2)$.
        Herein, $\Xns = \{X_2\}$, $\Yns = \{Y_2\}$, and $\Xs = \Ys = \varnothing$.
        Function \textbf{\s-ID} passes the condition in line 6 and defines $\Db = \{X_1, Y_1, Y_2\}$, leading to $\Db_{1} = \{X_1\}$ and $\Db_{2} = \{Y_1, Y_2\}$.
        It then defines $\Tb_i$'s in line 12 as $\Tb_{1} = \{X_1, X_2\}$ and $\Tb_{2} = \{Y_1, Y_2\}$.
        In line 13, it uses Lemma \ref{lemma: qs-decom} to compute $\Qs[\Tb_1] = \sum_{Y_1, Y_2} \Qs[\Vns]$ and $\Qs[\Tb_2] = \frac{\Qs[\Vns]}{\sum_{Y_1, Y_2} \Qs[\Vns]}$ (See Example \ref{ex: ex5}).
        It then calls Function \textbf{sID-Single}, which returns $\Qs[\Db_1] = \sum_{X_2} \Qs[\Tb_1]$ and $\Qs[\Db_2] = \Qs[\Tb_2]$.
        Finally, in line 20, the function returns
        \begin{equation*}
             \prs_{X_2}(Y_2) 
             = \sum_{Z_1, Z_2} \prs(Z_1, Z_2) \sum_{X_1,Y_1} \Qs[X_1]\Qs[Y_1, Y_2],
        \end{equation*}
        where $\Qs[X_1] = \prs(X_1 \vert Z_1, Z_2)$ and $\Qs[Y_1, Y_2] = \prs(Y_1, Y_2 \vert X_1, X_2, Z_1, Z_2)$.
    \end{example}

    \begin{example} \label{ex: ex9}
        Following the previous example, suppose we want to apply Algorithm \ref{algo: s-ID} for computing causal effect $\prs_{X_1}(Y_1, Y_2)$.
        In this case, the algorithm needs to compute $\Qs[Y_1, Y_2]$ and $\Qs[X_2]$.
        However, when the algorithm calls \textbf{sID-Single} $(X_2, \{X_1, X_2\}, \Qs[X_1, X_2])$, Function \textbf{sID-Single} returns \textsc{Fail}.
        Accordingly, Function \textbf{sID} returns \textsc{Fail} for $\prs_{X_1}(Y_1, Y_2)$.
    \end{example}
      
    \begin{remark}
        Algorithm \ref{algo: s-ID} is sound for the \s-ID problem in the presence of latent variables.
        We conjecture that this algorithm is also \emph{complete}, meaning that whenever it returns \textsc{Fail}, the corresponding causal effect is not \s-ID.
    \end{remark}

\section{Reduction from \s-Recoverability to \s-ID} \label{sec: reduction}
    
    Recall that the objective in \s-Recoverability is to compute $\pro_{\Xb}(\Yb)$ from $\prs(\Vb)$ \cite{Bareinboim_Tian_Pearl_2014}, while \s-ID aims to compute $\prs_{\Xb}(\Yb)$ from $\prs(\Vb)$.
    \cite{Bareinboim_Tian_2015} proposed RC, a sound algorithm for the \s-Recoverability problem.
    Subsequently, \cite{Correa_Tian_Bareinboim_2019} proved that RC is complete.
    In this section, we present a reduction from the \s-Recoverability problem to the \s-ID problem.
    This indicates that solving \s-ID can solve the \s-Recoverability problem (but not the other way around).
    
    \begin{theorem} \label{thm: reduction}
        For disjoint subsets $\Xb$ and $\Yb$ of $\Vb$, $\pro_{\Xb}(\Yb)$ can be uniquely computed from $\prs(\Vb)$ in the augmented ADMG $\Grs$ if and only if 
        \begin{equation} \label{eq: reduction}
            (\Yb \independent S \vert \Xb)_{\Grs_{\overline{\Xb}}},
        \end{equation}
        and $\prs_{\Xb}(\Yb)$ is \s-ID in $\Grs$.
    \end{theorem}
    \begin{remark} \label{remark: reduction}
        Equation \eqref{eq: reduction} is a very restrictive condition, and when it holds, Rule 1 of do-calculus implies that $\pro_{\Xb}(\Yb) = \prs_{\Xb}(\Yb)$.
    \end{remark}

    Theorem \ref{thm: reduction} implies that when Equation \eqref{eq: reduction} does not hold, then $\pro_{\Xb}(\Yb)$ is not \s-Recoverable.
    However, this causal effect might be identifiable in the target sub-population, i.e., $\prs_{\Xb}(\Yb)$ is \s-ID.   
    
    As a consequence of Theorem \ref{thm: reduction}, we propose Algorithm \ref{algo: reduction} for computing $\pro_{\Xb}(\Yb)$ from $\prs(\Vb)$.
    The algorithm takes as input two disjoint subsets $\Xb$ and $\Yb$ of $\Vb$ along with an augmented ADMG over $\Vb \cup \{S\}$ and the conditional distribution $\prs(\Vb)$.
    It first checks Equation \eqref{eq: reduction} in line 3, and then calls Algorithm \ref{algo: s-ID} as a subroutine to compute $\prs_{\Xb}(\Yb)$ from $\prs(\Vb)$ when it is \s-ID in $\Grs$.

    \begin{figure}[t]
         \begin{minipage}[t]{0.45\textwidth}
            \begin{algorithm}[H]
               \caption{\\ Reduction from \s-Recoverability to \s-ID}
               \label{algo: reduction}
                \begin{algorithmic}[1]
                    \STATE \textbf{Input:} $\Xb, \Yb, \Grs, \prs$
                    \STATE \textbf{Output:} Expression for $\pro_{\Xb}(\Yb)$ in terms of $\prs$ or \textsc{Fail}
                    \IF{$(\Yb \notindependent S \vert \Xb)_{\Grs_{\overline{\Xb}}}$}
                        \STATE \textbf{Return} \textsc{Fail}
                    \ELSE
                        \STATE \textbf{Return  sID}$(\Xb, \Yb, \Grs, \prs)$
                    \ENDIF
                \end{algorithmic}
            \end{algorithm}
        \end{minipage}
        \hfill
        \begin{minipage}[t]{0.6\textwidth}
            \begin{figure}[H]
                \centering
                \begin{tikzpicture}
                    \node [block, label=center:$X_1$](X1) {};
                    \node [block, below = 1.2 of X1, label=center:$X_2$](X2) {};
                    \node [block, right= 1.2 of X2, label=center:$Z_2$](Z2) {};
                    \node [block, right= 1.2 of X1, label=center:$Z_1$](Z1) {};
                    \node [block, left = 1.2 of X2, label=center:$Y$](Y) {};
                    \node [block-s, right = 1.2 of Z2, label=center:S](S) {};
                    \draw[edge] (X1) to (X2);
                    \draw[edge] (X2) to (Y);
                    \draw[edge] (Z1) to (Z2);
                    \draw[edge] (Z2) to (S);
            
                    \draw[edge2] (Z2) to  (X2);
                    \draw[edge2] (X1) to [bend right = 30]  (Y);
                    \draw[edge2] (Z1) to  (X1);
                \end{tikzpicture}
                \caption{The augmented ADMG $\Grs$ in Example \ref{ex: ex10}.}
                \label{fig: reduction}
            \end{figure}
        \end{minipage}
     \end{figure}

    \begin{example} \label{ex: ex10}
        Consider the augmented ADMG $\Grs$ in Figure \ref{fig: reduction}.
        In this graph, $(Y \notindependent S \vert X_1)_{\Grs_{\overline{X_1}}}$, thus, Theorem \ref{thm: reduction} implies that $\pro_{X_1}(Y)$ cannot be uniquely computed from $\prs(\Vb)$.
        On the other hand, $\pro_{X_2}(Y)$ can be identified from $\prs(\Vb)$ since $(Y \independent S \vert X_2)_{\Grs_{\overline{X_2}}}$ and due to Theorem \ref{thm: main}, $\prs_{X_2}(Y)$ is \s-ID in $\Grs$.
        In this case, Algorithm \ref{algo: reduction} returns the following expression for $\pro_{X_2}(Y)$ in terms of $\prs$
        \begin{equation*}
            \pro_{X_2}(Y)
            = \sum_{Z_1, Z_2} \prs(Z_1, Z_2) \sum_{X_1} \frac{\prs(X_1, X_2, Y \vert Z_1, Z_2)}{\prs(X_2 \vert X_1, Z_1, Z_2)}.
        \end{equation*}
    \end{example}

\section{Conclusion}
    The \s-ID problem, introduced by \cite{abouei2023sid}, asks whether, given the causal graph, a causal effect in a sub-population can be identified from the observational distribution pertaining to the same sub-population.
    \cite{abouei2023sid} addressed this problem when all the variables in the causal graph are observable.
    In this paper, we studied the \s-ID problem in the presence of latent variables and provided a sufficient graphical condition to determine whether a causal effect is \s-ID.
    Consequently, we proposed a sound algorithm for \s-ID.
    While this paper proves the soundness of our proposed method, we also conjecture that our approach is not only sound but also complete.
    Finally, by presenting an appropriate reduction, we showed that solving \s-ID can solve the \s-Recoverability problem.

\section*{Acknowledgments}
    We thank the anonymous reviewers for their feedback. This research was in part supported by the Swiss National Science Foundation under NCCR Automation, grant agreement 51NF40\_180545 and Swiss SNF project 200021\_204355 /1.

\bibliography{references}

\begin{thebibliography}{KMEK22}

\bibitem[AMK24]{abouei2023sid}
Amir~Mohammad Abouei, Ehsan Mokhtarian, and Negar Kiyavash.
\newblock s-id: Causal effect identification in a sub-population.
\newblock {\em Proceedings of the AAAI Conference on Artificial Intelligence}, 38(18):20302--20310, Mar. 2024.

\bibitem[BMS20]{bhattacharya2020causal}
Rohit Bhattacharya, Daniel Malinsky, and Ilya Shpitser.
\newblock Causal inference under interference and network uncertainty.
\newblock In {\em Uncertainty in Artificial Intelligence}, pages 1028--1038. PMLR, 2020.

\bibitem[BP12]{pmlr-v22-bareinboim12}
Elias Bareinboim and Judea Pearl.
\newblock Controlling selection bias in causal inference.
\newblock In Neil~D. Lawrence and Mark Girolami, editors, {\em Proceedings of the Fifteenth International Conference on Artificial Intelligence and Statistics}, volume~22 of {\em Proceedings of Machine Learning Research}, pages 100--108, La Palma, Canary Islands, 21--23 Apr 2012. PMLR.

\bibitem[BT15]{Bareinboim_Tian_2015}
Elias Bareinboim and Jin Tian.
\newblock Recovering causal effects from selection bias.
\newblock {\em Proceedings of the AAAI Conference on Artificial Intelligence}, 29(1), Mar. 2015.

\bibitem[BTP14]{Bareinboim_Tian_Pearl_2014}
Elias Bareinboim, Jin Tian, and Judea Pearl.
\newblock Recovering from selection bias in causal and statistical inference.
\newblock {\em Proceedings of the AAAI Conference on Artificial Intelligence}, 28(1), Jun. 2014.

\bibitem[CLB21]{correa2021nested}
Juan Correa, Sanghack Lee, and Elias Bareinboim.
\newblock Nested counterfactual identification from arbitrary surrogate experiments.
\newblock {\em Advances in Neural Information Processing Systems}, 34:6856--6867, 2021.

\bibitem[CTB19]{Correa_Tian_Bareinboim_2019}
Juan~D. Correa, Jin Tian, and Elias Bareinboim.
\newblock Identification of causal effects in the presence of selection bias.
\newblock {\em Proceedings of the AAAI Conference on Artificial Intelligence}, 33(01):2744--2751, Jul. 2019.

\bibitem[HV06]{huang2006identifiability}
Yimin Huang and Marco Valtorta.
\newblock Identifiability in causal bayesian networks: A sound and complete algorithm.
\newblock In {\em AAAI}, pages 1149--1154, 2006.

\bibitem[JAK23]{jamshidi2023causal}
Fateme Jamshidi, Sina Akbari, and Negar Kiyavash.
\newblock Causal imitability under context-specific independence relations.
\newblock In A.~Oh, T.~Naumann, A.~Globerson, K.~Saenko, M.~Hardt, and S.~Levine, editors, {\em Advances in Neural Information Processing Systems}, volume~36, pages 26810--26830. Curran Associates, Inc., 2023.

\bibitem[JRZB22]{NEURIPS2022_17a9ab41}
Amin Jaber, Adele Ribeiro, Jiji Zhang, and Elias Bareinboim.
\newblock Causal identification under markov equivalence: Calculus, algorithm, and completeness.
\newblock In S.~Koyejo, S.~Mohamed, A.~Agarwal, D.~Belgrave, K.~Cho, and A.~Oh, editors, {\em Advances in Neural Information Processing Systems}, volume~35, pages 3679--3690. Curran Associates, Inc., 2022.

\bibitem[JZB19]{pmlr-v97-jaber19a}
Amin Jaber, Jiji Zhang, and Elias Bareinboim.
\newblock Causal identification under {M}arkov equivalence: Completeness results.
\newblock In Kamalika Chaudhuri and Ruslan Salakhutdinov, editors, {\em Proceedings of the 36th International Conference on Machine Learning}, volume~97 of {\em Proceedings of Machine Learning Research}, pages 2981--2989. PMLR, 09--15 Jun 2019.

\bibitem[KEK23]{kivva2023identifiability}
Yaroslav Kivva, Jalal Etesami, and Negar Kiyavash.
\newblock On identifiability of conditional causal effects.
\newblock {\em The 39th Conference on Uncertainty in Artificial Intelligence}, 2023.

\bibitem[KMEK22]{kivva2022revisiting}
Yaroslav Kivva, Ehsan Mokhtarian, Jalal Etesami, and Negar Kiyavash.
\newblock Revisiting the general identifiability problem.
\newblock In {\em Uncertainty in Artificial Intelligence}, pages 1022--1030. PMLR, 2022.

\bibitem[LCB19]{Lee2019GeneralIW}
Sanghack Lee, Juan~David Correa, and Elias Bareinboim.
\newblock General identifiability with arbitrary surrogate experiments.
\newblock In {\em Conference on Uncertainty in Artificial Intelligence}, 2019.

\bibitem[LGS24]{lee2024a}
Jaron~J.R. Lee, AmirEmad Ghassami, and Ilya Shpitser.
\newblock A general identification algorithm for data fusion problems under systematic selection.
\newblock In {\em The 40th Conference on Uncertainty in Artificial Intelligence}, 2024.

\bibitem[MJEK22]{mokhtarian2022causal}
Ehsan Mokhtarian, Fateme Jamshidi, Jalal Etesami, and Negar Kiyavash.
\newblock Causal effect identification with context-specific independence relations of control variables.
\newblock In {\em International Conference on Artificial Intelligence and Statistics}, pages 11237--11246. PMLR, 2022.

\bibitem[Pea95]{pearl1995causal}
Judea Pearl.
\newblock Causal diagrams for empirical research.
\newblock {\em Biometrika}, 82(4):669--688, 1995.

\bibitem[Pea00]{pearl2000models}
Judea Pearl.
\newblock Causality: Models, reasoning and inference.
\newblock {\em Cambridge, UK: CambridgeUniversityPress}, 19(2):3, 2000.

\bibitem[Pea09]{pearl2009causality}
Judea Pearl.
\newblock {\em Causality}.
\newblock Cambridge university press, 2009.

\bibitem[PM18]{pearl2018book}
Judea Pearl and Dana Mackenzie.
\newblock {\em The book of why: the new science of cause and effect}.
\newblock Basic books, 2018.

\bibitem[Rub74]{rubin1974estimating}
Donald~B Rubin.
\newblock Estimating causal effects of treatments in randomized and nonrandomized studies.
\newblock {\em Journal of educational Psychology}, 66(5):688, 1974.

\bibitem[SGSH00]{spirtes2000causation}
Peter Spirtes, Clark~N Glymour, Richard Scheines, and David Heckerman.
\newblock {\em Causation, prediction, and search}.
\newblock MIT press, 2000.

\bibitem[SP06a]{shpitser2012identification}
Ilya Shpitser and Judea Pearl.
\newblock Identification of conditional interventional distributions.
\newblock {\em Proceedings of the 22nd Conference on Uncertainty in Artificial Intelligence}, 2006.

\bibitem[SP06b]{shpitser2006identification}
Ilya Shpitser and Judea Pearl.
\newblock Identification of joint interventional distributions in recursive semi-markovian causal models.
\newblock In {\em Proceedings of the National Conference on Artificial Intelligence}, volume~21, page 1219. Menlo Park, CA; Cambridge, MA; London; AAAI Press; MIT Press; 1999, 2006.

\bibitem[SS18]{sherman2018identification}
Eli Sherman and Ilya Shpitser.
\newblock Identification and estimation of causal effects from dependent data.
\newblock {\em Advances in neural information processing systems}, 31, 2018.

\bibitem[THK19]{tikka2019identifying}
Santtu Tikka, Antti Hyttinen, and Juha Karvanen.
\newblock Identifying causal effects via context-specific independence relations.
\newblock {\em Advances in neural information processing systems}, 32, 2019.

\bibitem[THK21]{Tikka_2021}
Santtu Tikka, Antti Hyttinen, and Juha Karvanen.
\newblock Causal effect identification from multiple incomplete data sources: A general search-based approach.
\newblock {\em Journal of Statistical Software}, 99(5), 2021.

\bibitem[TP02]{tian2002general}
Jin Tian and Judea Pearl.
\newblock A general identification condition for causal effects.
\newblock In {\em Proceedings of the Eighteenth National Conference on Artificial Intelligence (AAAI 2002)}, pages 567--573, Menlo Park, CA, 2002. AAAI Press/The MIT Press.

\bibitem[TP03]{tian2003ID}
Jin Tian and Judea Pearl.
\newblock On the identification of causal effects.
\newblock Technical report, Department of Computer Science, University of California, 2003.

\bibitem[ZMP23]{zhang2023causal}
Chi Zhang, Karthika Mohan, and Judea Pearl.
\newblock Causal inference with non-iid data under model uncertainty.
\newblock {\em Proceedings of Machine Learning Research vol TBD}, 1:14, 2023.

\end{thebibliography}
\bibliographystyle{alpha}


\newpage

\appendix

\begin{center} \label{sec: apd}
    {\Large \textbf{Appendix}}
\end{center}
The structure of the appendix is as follows. Appendix \ref{apd: example} includes an additional example of \s-ID problem in the presence of latent variables. In Appendix \ref{apd: lemmas}, we provide some preliminary lemmas used throughout our proofs. The proofs for the main results, namely, Lemmas \ref{lemma: qs-margin}, \ref{lemma: qs-decom}, and Theorems \ref{thm: main}, \ref{thm: reduction} are presented in Appendix \ref{apd: main}. In Appendix \ref{apd: exp}, we will conduct an experiment to compare the outputs of \s-ID algorithm and the classic ID algorithm.

\section{Additional Example} \label{apd: example}
\begin{figure}[h!]
    \centering

    \begin{tikzpicture}
            \tikzstyle{block-dashed} = [draw, fill=white, dashed, circle, text centered, inner sep=0.25cm]
                \node [block-dashed, label=center:$U_1$](U1) {};           
                \node [block-dashed, below = 0.4 of U1, label=center:$U_2$](U2) {};    
                \node [block-s, right= 1 of U2, label=center:$S$](S) {};
                \node [block, left = 1 of U2, label=center:$Y$](Y) {};
                \node [block, above = 1.5 of Y, label=center:$X$](X) {};   
                \draw[edge] (X) to (Y);
                \draw[edge,dashed] (U1) to (S); 
                \draw[edge, dashed] (U1) to (X);  
                \draw[edge, dashed] (U2) to (S);
                \draw[edge, dashed] (U2) to (Y);     
    \end{tikzpicture}
    \caption{ADMG $\Grs$ of the example of Appendix \ref{apd: example}.}
    \label{fig: ex11}          
\end{figure}

    In this section, we provide an example where ignoring $S$ results in an identifiable effect, whereas the target causal effect is, in fact, not \s-ID.
    
    Consider the following two SCMs.
    
    SCM $\M_1$:
    \begin{align*}
        U_1 & \sim Bern(0.5) \\
        U_2 & \sim Bern(0.5) \\
        \varepsilon_y & \sim Bern(0.3) \\
        X &= U_1 \\
        Y &= X \oplus U_2 \oplus \varepsilon_y  \\
        S &= \overline{U_1 \oplus U_2}
    \end{align*}
    where $\oplus$ denotes the XOR operator, and $Bern(p)$ denotes a Bernoulli random variable with parameter $p$.
    
    SCM $\M_2$:
    \begin{align*}
        U_1 & \sim Bern(0.5) \\
        U_2 & \sim Bern(0.5) \\
        \varepsilon_y & \sim Bern(0.3) \\
        X &= U_1 \\
        Y &= \varepsilon_y  \\
        S &= 1
    \end{align*}
    According to the above equations, we have
    \begin{equation*}
        \pro^{\M_1}(X = x, Y = y \vert S= 1) = \pro(U_1 = x) \pro(\varepsilon_y = y) = 0.5 \times \pro(\varepsilon_y = y) > 0. 
    \end{equation*}
    Similarly for $\M_2$ we have
    \begin{align*}
        \pro^{\M_2}(X = x, Y = y \vert S = 1) 
        &=  \pro^{\M_2}(X = x, Y = y)   \\
        &= \pro^{\M_2}(X = x) \pro^{\M_2}(Y = y)  \\
        &=  \pro(U_1 = x) \pro(\varepsilon_y = y) = 0.5 \times \pro(\varepsilon_y = y) > 0. 
    \end{align*}
    Thus, $\pro^{\M_1}(X, Y \vert S= 1) = \pro^{\M_2}(X, Y \vert S= 1) > 0$.
    Furthermore, we have
    \begin{equation*}
        \pro^{\M_1}_{x = 0}(Y = 1 \vert S = 1) =  \pro^{\M_1}_{x = 0}(U_2 + \varepsilon_y = 1 \vert S = 1) = \pro^{\M_1}(U_2 + \varepsilon_y = 1)  = 0.5.
    \end{equation*}
    Similarly, for SCM $\M_2$:
    \begin{equation*}
        \pro^{\M_2}_{x = 0}(Y = 1 \vert S = 1) = \pro^{\M_2}(\varepsilon_y = 1 \vert S = 1) = \pro(\varepsilon_y = 1) = 0.3
    \end{equation*}
    This shows that $\prs_{X}(Y)$ is not \s-ID in this causal graph, as $\pro^{\M_1}(X, Y \vert S= 1) = \pro^{\M_2}(X, Y \vert S= 1) > 0$, but $\pro^{\M_1}_{x = 0}(Y = 1 \vert S = 1) \neq \pro^{\M_2}_{x = 0}(Y = 1 \vert S = 1)$.

    Note that ignoring the sub-population, the causal effect $P_{X}(Y)$ is clearly identifiable from $P(X, Y)$, but as we showed above, the causal effect of $X$ on $Y$ is not identifiable in the sub-population from the observational data of that sub-population.

\section{Technical Preliminaries} \label{apd: lemmas}

    \textbf{Pearl's do-calculus rules \cite{pearl2000models}}: Let $\Xb,\Yb, \Zb, \Wb$ be four disjoint subsets of $\Vb$. The following three rules, commonly referred to as Pearl's do-calculus rules \cite{pearl2000models}, provide a tool for calculating interventional distributions using the causal graph.
        \begin{itemize}
            \item \textbf{Rule 1}: 
                If $(\Yb \independent \Zb \vert \Xb, \Wb)_{\Gr_{\overline{X}}}$, then
               \begin{equation*}
                   \pro_{\Xb}(\Yb \vert \Zb, \Wb) = \pro_{\Xb}(\Yb \vert \Wb).
               \end{equation*}    
            \item \textbf{Rule 2}: 
                If $(\Yb \independent \Zb \vert \Xb, \Wb)_{\Gr_{\overline{\Xb}\underline{\Zb}}}$, then
                \begin{equation*}
                      \pro_{\Xb, \Zb}(\Yb \vert \Wb) = \pro_{\Xb}(\Yb \vert \Zb, \Wb). 
                \end{equation*}  
            \item \textbf{Rule 3}: 
                If $(\Yb \independent \Zb \vert \Xb, \Wb)_{\Gr_{\overline{\Xb}\overline{\Zb(W)}}}$, where $\Zb(\Wb) \coloneqq \Zb  \setminus \Anc{\Wb}{\Gr_{\overline{\Xb}}}$, then
                \begin{equation*}
                    \pro_{\Xb, \Zb}(\Yb \vert \Wb) = \pro_{\Xb} (\Yb \vert \Wb). 
                \end{equation*}
    \end{itemize}
    
    \begin{lemma}[\citealp{tian2003ID}] \label{lemma:q_margin}
        For two sets $\Wb' \subset \Wb$, if $\Wb'$ is an ancestral set in $G[\Wb]$, then
        \begin{equation} \label{eq:margin}
            Q[\Wb'] = \mysum{\Wb \setminus \Wb'}{} Q[\Wb].
        \end{equation}
    \end{lemma}

    \begin{lemma}[\citealp{tian2003ID}] \label{lemma:q_decom}
        Let $\Cb \subseteq \Vb$, and assume that C is partitioned into \ccos $\Cb_1$, \dots, $\Cb_{m}$ in the subgraph $G[\Cb]$. Then we have
        \begin{itemize}
            \item
                $Q[\Cb]$ decomposes as 
                \begin{equation}
                    \label{eq:q_lemma_decom}
                  Q[\Cb] = \prod_{i=1}^{m} Q[\Cb_{i}]  
                \end{equation}
            \item
                Let $k$ denote the number of variables in $\Hb$, and let us assume a topological order of variables in $\Cb$ as $V_{c_1} < V_{c_2} < \dots < V_{c_k}$ in $G_{\Cb}$.
                Let $\Cb^{i}$ be the set of variables in $\Cb$ ordered before $V_{c_i}$ (including $V_{c_i}$), for $i = 1, 2, \dots, k$, where $\Cb^0$ is an empty set.
                Then each $Q[\Cb_{j}]$, $j = 1, 2, \dots, m$, is computable from $Q[\Cb]$ and is given by
                \begin{equation}
                \label{eq:q_lemma_prod}    
                    Q[\Cb_{j}] = \prod_{\{i\vert V_{c_{i}} \in \Cb_{j} \}} \frac{Q[\Cb^{(i)}]}{Q[\Cb^{(i-1)}]},
                \end{equation}
                where each $Q[\Cb^{i}]$ is given by
                \begin{equation}    
                    Q[\Cb^{(i)}] = \mysum{\Cb \setminus \Cb^{(i)}}{} Q[\Cb].
                \end{equation}
        \end{itemize}
    \end{lemma}
    
    \begin{corollary}
        \label{cor: q-decom}
        Let $\Hb_1 \sqcup \Hb_2 \sqcup \dots \sqcup \Hb_m$ be a partition of set $\Cb$, where for each $\Cb_{i}$ there exists $1 \leq j \leq m$ such that $ \Cb_i \subseteq \Hb_{j}$.
        Then, we have
        \begin{equation*}
            Q[\Cb] = \prod_{j} Q[\Hb_j].
        \end{equation*}
    \end{corollary}

    \begin{lemma} \label{lemma: Qs-eq-q}
       Let $\Hb$ be a subset of $\Vns$, we have the following equation
        \begin{equation} \label{eq: Qs-eq-q}
            \Qs[\Hb] = \frac{Q[\Hb \cup \Anc{S}{}]}{Q[\Anc{S}{}]}.
        \end{equation}
    \end{lemma}
    \begin{proof}
        According to definition of $\Qs$, we have
        $$
            \Qs[\Hb] = \pro_{\Vns \setminus \Hb}(\Hb \vert \Anc{S}{}) = \frac{\pro_{\Vns \setminus \Hb}(\Hb , \Anc{S}{})}{\pro_{\Vns \setminus \Hb}(\Anc{S}{})} = \frac{Q[\Hb \cup \Anc{S}{}]}{\pro_{\Vns \setminus \Hb}(\Anc{S}{})}.
        $$
        Note that $\Anc{S}{}$ is an ancestral set in $\Gr$; Hence, for any $\Hb$, we have
        $$
            \pro_{\Vns \setminus \Hb}(\Anc{S}{})  = \pro(\Anc{S}{})  = Q[\Anc{S}{}].
        $$
        This implies Equation \eqref{eq: Qs-eq-q}.
    \end{proof}

    \begin{lemma} \label{lemma: inter-cond}
        Let $\Xns$ and $\Yns$ be disjoint subsets of $\Vns = \Vb \setminus \Anc{S}{\Grs}$.
        We have
        \begin{equation*}
            \pro_{\Xns}(\Yns \vert \Anc{S}{}) = \mysum{\Db \setminus \Yns}{} \Qs[\Db_1]\dots \Qs[\Db_k],
        \end{equation*}
        where $\Db = \Anc{\Yns}{\Gr[\Vns \setminus \Xns]}$, and $\Db_i$'s are \scos of $\Db$.
    \end{lemma}

    \begin{proof}
        Using marginalization over $\Vns \setminus (\Xns \cup \Yns)$, we have
        $$
           \pro_{\Xns}(\Yns \vert \Anc{S}{}) = \mysum{\Vns \setminus (\Xns \cup \Yns)}{}\pro_{\Xns} (\Vns \setminus \Xns \vert \Anc{S}{}) = \mysum{\Vns \setminus (\Xns \cup \Yns)}{} \Qs[\Vns \setminus \Xns].
        $$
        Since $\Db$ is an ancestral set in $\Gr[\Vns \setminus \Xns]$, according to Lemma \ref{lemma: qs-margin}, we have 
        $$
         \pro_{\Xns}(\Yns \vert \Anc{S}{}) = \mysum{\Vns \setminus (\Xns \cup \Yns)}{} \Qs[\Vns \setminus \Xns] = \mysum{\Db \setminus \Yns}{} \mysum{\Vns \setminus (\Xns \cup \Db)}{}
         \Qs[\Vns \setminus \Xns] = \mysum{\Db \setminus \Yns}{} Q[\Db].
        $$
        Therefore, the first property of Lemma \ref{lemma: qs-decom} implies that
        \begin{equation*}
            \pro_{\Xns}(\Yns \vert \Anc{S}{}) = \mysum{\Db \setminus \Yns}{} \Qs[\Db_1]\dots \Qs[\Db_k].
        \end{equation*}
    \end{proof}

\begin{lemma}
        \label{lemma: algorithm-hedge}
        If Function sID-Single returns \textsc{Fail} for the inputs $\Cb$ and $\Tb$, then $\Tb$ is an \s-Hedge for $\Cb$.
    \end{lemma}
    \begin{proof}
        When the algorithm returns Fail, 
        \begin{enumerate}
            \item $\Cb$ and $\Tb$ are both single \scos since the inputs of the algorithm have to be \scos.
            \item If $\Ab \coloneqq \Anc{\Cb}{\Gr[\Tb]}$, then $\Tb = \Ab$.
        \end{enumerate}
        According to the definition of s-Hedge \ref{def: s-Hedge}, $\Tb$ is an s-Hedge for $\Cb$.
    \end{proof}

\section{Proofs of Main Results} \label{apd: main}
    \begingroup
    \def\thetheorem{\ref{lemma: qs-margin}}
        \begin{lemma}   
            Let $\Wb, \Wb'$ be two subsets of $\Vns$ such that $\Wb' \subset \Wb$. If $\Wb'$ is an ancestral set in $G[\Wb]$, then we have
            \begin{equation}
                \Qs[\Wb'] = \mysum{\Wb \setminus \Wb'}{} \Qs[\Wb].
            \end{equation}
        \end{lemma}
    \addtocounter{theorem}{-1}
    \endgroup

    \begin{proof}
        According to Lemma \ref{lemma: Qs-eq-q}, we have the following equations
        \begin{align*}
            \Qs[\Wb] &= \frac{Q[\Wb \cup \Anc{S}{}]}{Q[\Anc{S}{}]}, \\
            \Qs[\Wb'] &= \frac{Q[\Wb' \cup \Anc{S}{}]}{Q[\Anc{S}{}]}.
        \end{align*}
        Therefore, by replacing the above equations in Equation \eqref{eq: qs-margin}, it is sufficient to show that
        \begin{equation*}
            \frac{Q[\Wb' \cup \Anc{S}{}]}{Q[\Anc{S}{}]} = \mysum{\Wb \setminus \Wb'}{} \frac{Q[\Wb \cup \Anc{S}{}]}{Q[\Anc{S}{}]}.
        \end{equation*}
        Since $\pro(\Anc{S}{}) > 0$, this is equivalent to
        $$
          Q[\Wb' \cup \Anc{S}{}] =   \mysum{\Wb \setminus \Wb'}{} Q[\Wb \cup \Anc{S}{}].
        $$
        Note that $\Wb' \cup \Anc{S}{}$ in an ancestral set in $\Gr[\Wb \cup \Anc{S}{}]$, and $\Wb \setminus \Wb' = (\Wb \cup \Anc{S}{}) \setminus (\Wb' \cup \Anc{S}{})$.
        Hence, Lemma \ref{lemma:q_margin} concludes the proof.
    
    \end{proof}

    \begingroup
    \def\thetheorem{\ref{lemma: qs-decom}}
         Suppose $\Hb \subseteq \Vns$ and let $\Hb_1, \dots, \Hb_k$ denote the \scos of $\Hb$ in $\Grs$. Then,
        \begin{itemize}
            \item $\Qs[\Hb]$ decomposes as
                \begin{equation*}
                    \Qs[\Hb] = \Qs[\Hb_1]\Qs[\Hb_2]\dots \Qs[\Hb_k].
                \end{equation*}
            \item Let $m$ be the number of variables in $\Hb$, and consider a topological ordering of  the variables in graph $\Grs[\Hb]$, denoted as $V_{h_1} < \dots < V_{h_m}$.
            Let $\Hb^{(0)} = \emptyset$ and for each $1 \leq i \leq  m$, $\Hb^{(i)}$ denote the set of variables in $\Hb$ ordered before $V_{h_i}$ (including $V_{h_i}$).
            For every $1\leq j \leq k$, $\Qs[\Hb_{j}]$ can be computed from $\Qs[\Hb]$ by   
            \begin{equation}
                \Qs[\Hb_{j}] = \prod_{\{i \vert V_{h_{i}} \in \Hb_{j} \}} \frac{\Qs[\Hb^{(i)}]}{\Qs[\Hb^{(i-1)}]},
            \end{equation}
            where $\Qs[\Hb^{(i)}]$s can be computed by
            \begin{equation}    
                 \Qs[\Hb^{(i)}] = \mysum{\Hb \setminus \Hb^{(i)}}{} \Qs[\Hb].
            \end{equation}
        \end{itemize}
    \addtocounter{theorem}{-1}
    \endgroup
    
    \begin{proof}
        \textbf{First part:} according to the definition of $\Qs[.]$, we have
        $$
                \Qs[\Hb] = \pro_{\Vns \setminus \Hb}(\Hb \vert \Anc{S}{}) = \frac{\pro_{\Vns \setminus \Hb}(\Hb , \Anc{S}{})}{\pro_{\Vns \setminus \Hb}(\Anc{S}{})} = \frac{Q[\Hb \cup \Anc{S}{}]}{Q[\Anc{S}{}]}.
        $$
        Note the last equality holds because $\Anc{S}{}$ is an ancestral set in $\Gr$.
        Now, we have
        $$
            Q[\Anc{S}{}] = \pro_{\Vb \setminus \Anc{S}{}}(\Anc{S}{}) = \pro(\Anc{S}{}) = \pro_{\Vns \setminus \Hb}(\Anc{S}{}).
        $$
        Similarly, for each $i \in [1:k]$, we have
        $$ 
                \Qs[\Hb_i] = \frac{Q[\Hb_i \cup \Anc{S}{}]}{Q[\Anc{S}{}]}
        $$
        Therefore, the above equations imply that
        \begin{equation}
            \label{eq: iff-q-decom}
        \Qs[\Hb] = \Qs[\Hb_1]\Qs[\Hb_2]\dots \Qs[\Hb_k] \iff \frac{Q[\Hb \cup \Anc{S}{}]}{Q[\Anc{S}{}]} = \frac{Q[\Hb_1 \cup \Anc{S}{}]}{Q[\Anc{S}{}]}\times \cdots \times\frac{Q[\Hb_k \cup \Anc{S}{}]}{Q[\Anc{S}{}]}
        \end{equation}
        Let $\Cb_i$'s be the corresponding \cco of $\Hb_i$ (i.e., $\Hb_i \subseteq \Cb_i$ and $\Cb_i$ is a \cco of $\Hb_i \cup \Anc{S}{}$).
        According to the definition of $\Cb_i$ and Corollary \ref{cor: q-decom} we have 
        $$
            Q[\Hb_i \cup \Anc{S}{}] = Q[\Cb_i] Q[\Anc{S}{} \setminus \Cb_{i}].
        $$
        Moreover, since $\Cb_i$ is a \cco of $\Anc{S}{} \cup \Hb_i$, then $\Cb_i \setminus \Hb_i$ should be union of some \ccos of $\Anc{S}{}$. Therefore, for each $i$, Corollary \ref{cor: q-decom} implies
        $$
            Q[\Anc{S}{}] = Q[\Cb_i \setminus \Hb_i]Q[\Anc{S}{} \setminus (\Cb_i \setminus \Hb_i)] = Q[\Cb_i \setminus \Hb_i]Q[\Anc{S}{} \setminus \Cb_i]
        $$
        Putting the above equations together, we have
        \begin{equation} \label{eq: Qs-Q-2}
           \Qs[\Hb_i] = \frac{Q[\Hb_i \cup \Anc{S}{}]}{Q[\Anc{S}{}]} = \frac{Q[\Cb_i]Q[\Anc{S}{} \setminus \Cb_i]}{Q[\Cb_i \setminus \Hb_i]Q[\Anc{S}{} \setminus \Cb_i]} =  \frac{Q[\Cb_i]}{Q[\Cb_i \setminus \Hb_i]}.
        \end{equation}
        Hence,
        $$
            \prod_{i = 1}^{k} \frac{Q[\Hb_i \cup \Anc{S}{}]}{Q[\Anc{S}{}]} = \prod_{i = 1}^{k} \frac{Q[\Cb_i]}{Q[\Cb_i \setminus \Hb_i]}.
        $$
        Consequently, we have
        $$    
            \Qs[\Hb] = \Qs[\Hb_1]\Qs[\Hb_2]\dots \Qs[\Hb_k] \iff  Q[\Hb \cup \Anc{S}{}] = \frac{Q[\Anc{S}{}]}{\prod_{i} Q[\Cb_i \setminus \Hb_i]}  \prod_{i = 1}^{k} Q[\Cb_i].
        $$
        Note that $\Cb_i \setminus \Hb_i$ are disjoint subsets of $\Anc{S}{}$, where each of them is the union of some \ccos of $\Anc{S}{}$. If $\Cb_{k + 1} \coloneqq \Anc{S}{} \setminus \bigcup_{i} \Cb_i$, then Corollary \ref{cor: q-decom} implies the following.
        $$
            Q[\Anc{S}{}] = \prod_{i = 1}^{k} Q[\Cb_i \setminus \Hb_i] Q[\Cb_{k + 1}]
            \Longrightarrow \frac{Q[\Anc{S}{}]}{\prod_{i = 1}^{k} Q[\Cb_i \setminus \Hb_i]} = Q[\Cb_{k + 1}]
        $$
        Finally, applying Corollary \ref{cor: q-decom} for $\Hb \cup \Anc{S}{}$, we have
        $$
             Q[\Hb \cup \Anc{S}{}] = Q[\Cb_{k + 1}] \prod_{i = 1}^{k} Q[\Cb_i].
        $$
        This concludes the first part.
        
        \textbf{Second part:} the proof is similar to proof of Lemma \ref{lemma:q_decom}.
        Define $\Cb \coloneqq \Hb \cup \Anc{S}{}$.
        Let $\Cb_{1}, \dots, \Cb_{m}$ be \ccos of $C$ (w.l.o.g. assume that $\Hb_i \subseteq \Cb_i$ for $i \in [1:k]$). Consider a topological order of $V_{h_1} < V_{h_{2}} < \cdots < V_{h_{n}}$ for $\Hb$, and $V_{s_1} < \cdots < V_{s_{n'}}$ for $\Anc{S}{}$. Since any $V_{h_i}$ is not an ancestor of $\Anc{S}{}$, $V_{s_1} < \cdots < V_{s_{n'}} < V_{h_1} < V_{h_{2}} < \cdots < V_{h_{n}}$ is a topological order for $\Hb \cup \Anc{S}{}$. According to Lemma \ref{lemma:q_decom}, we have
        $$
             Q[\Cb_{j}] = \prod_{\{i\vert V_{c_{i}} \in \Cb_{j} \}} \frac{Q[\Cb^{(i)}]}{Q[\Cb^{(i-1)}]},
        $$
        where $(c_1, c_2, \dots, c_{n + n'}) = (s_1,\dots, s_{n'}, h_1, \dots, h_n)$.
        For each $i \in [1:n]$, let $\Hb^{i} \coloneqq \{V_{h_1}, V_{h_2},\dots,V_{h_{i}}\}$ and $\Hb^{0} = \varnothing$.
        For each $i > 0$, we have
        $$
           \frac{Q[\Cb^{(i + n')}]}{Q[\Cb^{(i + n'-1)}]} = \frac{Q[\Hb^{i} \cup \Anc{S}{}]}{Q[\Hb^{(i-1)} \cup \Anc{S}{}]} = \frac{\frac{Q[\Hb^{i} \cup \Anc{S}{}]}{Q[\Anc{S}{}]}}{\frac{Q[\Hb^{(i-1)} \cup \Anc{S}{}]}{Q[\Anc{S}{}]}} = \frac{\Qs[\Hb^{i}]}{\Qs[\Hb^{(i-1)}]},
        $$
        where the last equality holds according to Lemma \ref{lemma: Qs-eq-q}.
        It suffices to show that
        $$
            \Qs[\Hb_{j}] = \prod_{\{i\vert V_{h_{i}} \in \Hb_{j} \}} \frac{\Qs[\Hb^{(i)}]}{\Qs[\Hb^{(i-1)}]}.
        $$
        Equation \ref{eq: Qs-Q-2} implies that
        $$ 
            \Qs[\Hb_{j}] = \frac{Q[\Cb_{j}]}{Q[\Cb_{j} \setminus \Hb_{j}]}.
        $$
        Hence, according to this equation and Lemma \ref{lemma:q_decom}, we have
        \begin{align}
        \label{eq: decom_eq2}
             Q[\Cb_{j} \setminus \Hb_{j}] \Qs[\Hb_{j}] = Q[\Cb_{j}] &= \prod_{\{i\vert V_{c_{i}} \in \Cb_{j} \}} \frac{Q[\Cb^{(i)}]}{Q[\Cb^{(i-1)}]} \nonumber \\
                                      &= \prod_{\{i\vert V_{c_{i}} \in {\Cb_{j} \cap \Anc{S}{}} \}} \frac{Q[\Cb^{(i)}]}{Q[\Cb^{(i-1)}]} \prod_{\{i\vert V_{h_{i}} \in \Hb_{j} \}} \frac{\Qs[\Hb^{(i)}]}{\Qs[\Hb^{(i-1)}]}  \nonumber \\
                                 &=  \prod_{\{i\vert V_{s_{i}} \in {\Cb_{j} \setminus \Hb_{j}} \}} \frac{Q[\Cb^{(i)}]}{Q[\Cb^{(i-1)}]} \prod_{\{i\vert V_{h_{i}} \in \Hb_{j} \}} \frac{\Qs[\Hb^{(i)}]}{\Qs[\Hb^{(i-1)}]}.
        \end{align}
    
        Based on the definition $\Cb_j$ and $\Hb_j$, $\Cb_j \setminus \Hb_j$ is the union of some \ccos of $\Anc{S}{}$.
        Denote the c-components of $\Cb_j \setminus \Hb_j$ by $A_1, \dots, A_t$.
        Applying Lemma \ref{lemma:q_decom} to $\Cb = \Anc{S}{}$ with topological order $V_{s_1} < \cdots < V_{s_{n'}}$, we have
        \begin{align*}
            Q[A_1] &= \prod_{\{i\vert V_{s_{i}} \in A_1 \}} \frac{Q[\Cb^{(i)}]}{Q[\Cb^{(i-1)}]} \\
                &~\vdots \\
            Q[A_t] &= \prod_{\{i\vert V_{s_{i}} \in A_t \}} \frac{Q[\Cb^{(i)}]}{Q[\Cb^{(i-1)}]}
        \end{align*}
        By multiplying all $Q[A_i]$, we have
        $$
            Q[\Cb_j \setminus \Hb_j] = Q[A_1 \cup A_2 \dots \cup A_t] = \prod_{t' = 1}^{t} Q[A_{t'}] = \prod_{\{i\vert V_{s_{i}} \in {\Cb_{j} \setminus \Hb_{j}} \}} \frac{Q[\Cb^{(i)}]}{Q[\Cb^{(i-1)}]}.
        $$
        By substituting this in Equation \ref{eq: decom_eq2}, we obtain
        $$
         \Qs[\Hb_{j}] = \prod_{\{i\vert V_{h_{i}} \in \Hb_{j} \}} \frac{\Qs[\Hb^{(i)}]}{\Qs[\Hb^{(i-1)}]}.
        $$
        
        To complete the proof, it suffices to show that $\Qs[\Hb^{(i)}]$ are computable from $\Qs[\Hb]$.
        Since we have used the topological order, $\Hb^{i}$ is an ancestral set in $\Gr[\Hb]$.
        Therefore, Lemma \ref{lemma: qs-margin} implies that
        $$
            \Qs[\Hb^{i}] = \mysum{\Hb\setminus \Hb^{i}}{} \Qs[\Hb].
        $$
        This shows that $\Qs[\Hb_i]$ is uniquely computable from $\Qs[\Hb]$. 
    \end{proof}

    \begingroup
        \def\thetheorem{\ref{thm: main}}
        \begin{theorem}
            For disjoint subsets $\Xb$ and $\Yb$ of $\Vb$, let 
            \begin{equation*}
                \Xs \coloneqq \Xb \cap \Vs, \quad \Xns \coloneqq \Xb \cap \Vns,~\text{and} ~ \Yns \coloneqq \Yb \cap \Vns.
            \end{equation*}
            \begin{enumerate}
                \item Conditional causal effect $\prs_{\Xb}(\Yb)$ is \s-ID in $\Grs$ if and only if
                \begin{equation}
                    (\Xs \independent \Yb \vert \Xns, S)_{\Grs_{\underline{\Xs}\overline{\Xns}}},
                \end{equation}
                and $\prs_{\Xns}(\Yb, \Xs)$ is \s-ID in $\Grs$.
                \item Suppose $\Db \coloneqq \Anc{\Yns}{\Grs[\Vns \setminus \Xns]}$ and let $\{\Db_i\}_{i=1}^k$ denote the \scos of $\Db$ in $\Grs$.
                Conditional causal effect $\prs_{\Xns}(\Yb, \Xs)$ is \s-ID in $\Grs$ if there are no \s-Hedge in $\Grs$ for any of $\{\Db_i\}_{i=1}^k$ .
            \end{enumerate}
        \end{theorem}
        \addtocounter{theorem}{-1}
    \endgroup

    \begin{proof}
        We first show that $(\Xs \independent \Yb \vert \Xns, S)_{\Grs_{\underline{\Xs}\overline{\Xns}}}$ is a necessary condition.
        Suppose $(\Xs \notindependent \Yb \vert \Xns, S)_{\Grs_{\underline{\Xs}\overline{\Xns}}}$.
        Denote by $\Gr'$ the equivalent DAG of $\Grs$, obtained by adding the unobserved variables. Theorem 2 in \cite{abouei2023sid} shows that there are two SEMs $\M_1$ and $\M_2$ compatible with $\Gr'$ such that
        $$   
           \pro^{\M_1}(\Vb, \Ub \vert S = 1) =  \pro^{\M_2}(\Vb, \Ub \vert S = 1), \text{and}
        $$
        $$
            \pro^{\M_1}_{\Xb}(\Yb \vert S = 1) \neq \pro^{\M_2}_{\Xb}(\Yb \vert S = 1).
        $$
        
        We use these SEMs to construct SCMs $\M'_1$ and $\M'_2$ compatible with $\Grs$, in which we have
        $$
        \pro^{\M'_1}(\Vb, \Ub \vert S = 1) =  \pro^{\M'_2}(\Vb, \Ub \vert S = 1) \Longrightarrow \pro^{\M'_1}(\Vb \vert S = 1) =  \pro^{\M'_2}(\Vb \vert S = 1).
        $$
        Note that all causal effects in both models are the same for $\Grs$ and $\Gr$.
        Hence, when $(\Xs \notindependent \Yb \vert \Xns, S)_{\Grs_{\underline{\Xs}\overline{\Xns}}}$ holds, then $\prs_{\Xb}(\Yb)$ is not \s-ID.
        It shows that this condition is a necessary condition.
        
        Now suppose that $(\Xs \independent \Yb \vert \Xns, S)_{\Gr_{\underline{\Xs}\overline{\Xns}}}$ holds.
        According to Rule 2 of do-calculus, we have 
        $$
            \pro_{\Xb}(\Yb \vert S = 1) =  \pro_{\Xns}(\Yb \vert \Xs, S = 1)
        $$
        The above equation shows that $\pro_{\Xb}(\Yb \vert S = 1)$ is s-ID in $\Grs$ if and only if $\pro_{\Xns}(\Yb \vert \Xs, S = 1)$ is s-ID in $\Grs$.
        Moreover,
        $$
            \pro_{\Xns}(\Yb \vert \Xs, S = 1) = \frac{\pro_{\Xns}(\Yb, \Xs \vert S = 1)}{\pro_{\Xns}(\Xs \vert S = 1)}.
        $$
        Since $\Xns \cap \Vs = \varnothing$, acccording to Rule 3 of do-calculus, we have $\pro_{\Xns}(\Xs \vert S = 1) = \pro(\Xs \vert S = 1)$.
        Hence,
        \begin{equation} \label{eq: appendix-main-theorem}
            \pro_{\Xns}(\Yb \vert \Xs, S = 1) =  \frac{\pro_{\Xns}(\Yb, \Xs \vert S = 1)}{\pro(\Xs \vert S = 1)}.
        \end{equation}
        This shows that \s-Identifiability of $\pro_{\Xns}(\Yb \vert \Xs, S = 1)$ and $\pro_{\Xns}(\Yb, \Xs \vert S = 1)$ are equivalent (note that $\pro(\Xs \vert S = 1) > 0$ due to the positivity assumption in Definition \ref{def:S-ID}).

        Let $\Wb \coloneqq \Vs \setminus (\Xs \cup \Ys)$, then 
        \begin{align*}
            \pro_{\Xns} (\Yb \vert \Xs, S = 1)
            &= \mysum{\Wb}{} \pro_{\Xns} (\Yb, \Wb \vert \Xs, S = 1) \\
            &= \mysum{\Wb}{} \pro_{\Xns} (\Ys, \Wb \vert \Xs, S = 1) \pro_{\Xns}(\Yns \vert \Xs, \Ys, \Wb, S = 1) \\
            &= \mysum{\Wb}{} \pro_{\Xns} (\Ys, \Wb \vert \Xs, S = 1) \pro_{\Xns}(\Yns \vert \Vs, S = 1).
        \end{align*}

        The above equalities have been obtained by a marginalization over $\Wb$, chain rule, and replacing the definition of $\Wb$, respectively.
        According to Rule 3 of do-calculus, since $\Xns \cap \Vs = \varnothing$, we have 
        \begin{equation*}
            \pro_{\Xns} (\Ys, \Wb \vert \Xs, S = 1)  =  \pro (\Ys, \Wb \vert \Xs, S = 1).
        \end{equation*}
        Moreover, according to Lemma \ref{lemma: inter-cond}, we have
        $$
            \pro_{\Xns}(\Yns \vert \Vs, S = 1) = \mysum{\Db \setminus \Yns}{} \Qs[\Db_1]\dots \Qs[\Db_k].
        $$
        Combining the aforementioned equations with Equation \eqref{eq: appendix-main-theorem}, we get
        \begin{equation}
            \pro_{\Xns} (\Yb, \Xs \vert S = 1) = \mysum{\Wb}{} \pro (\Xs, \Ys, \Wb \vert S = 1) \mysum{\Db \setminus \Yns}{} \Qs[\Db_1]\dots \Qs[\Db_k].
        \end{equation}
        Now, note that according to Lemma \ref{lemma: algorithm-hedge}, if there is not any \s-Hedge for $\Db_i$s, then Function sID-Single will compute $\Qs[\Db_i]$s.
        Therefore, we can compute $\pro_{\Xns} (\Yb, \Xs \vert S = 1)$ using the above equation, which concludes the proof.
        As a result, if Equation \eqref{eq: main thm} holds, we have
        \begin{equation} \label{eq: appendix-main-th-eq2}
            \prs_{\Xb}(\Yb) =  \prs_{\Xns}(\Yb \vert \Xs) = \mysum{\Wb}{} \pro (\Ys, \Wb \vert \Xs, S = 1) \mysum{\Db \setminus \Yns}{} \Qs[\Db_1]\dots \Qs[\Db_k].
        \end{equation}
    \end{proof}

    \begingroup
    \def\thetheorem{\ref{thm: reduction}}
        \begin{theorem} 
            For disjoint subsets $\Xb$ and $\Yb$ of $\Vb$, $\pro_{\Xb}(\Yb)$ can be uniquely computed from $\prs(\Vb)$ in the augmented ADMG $\Grs$ if and only if 
            \begin{equation*}
                (\Yb \independent S \vert \Xb)_{\Grs_{\overline{\Xb}}},
            \end{equation*}
            and $\prs_{\Xb}(\Yb)$ is \s-ID in $\Grs$.
        \end{theorem}
    \addtocounter{theorem}{-1}
    \endgroup
    
    \begin{proof}
        If $(\Yb \notindependent S \vert \Xb)_{\Gr_{\overline{\Xb}}}$, then \cite[Theorem 2]{Bareinboim_Tian_2015} implies that $\pro_{\Xb}(\Yb)$ is not \s-recoverable.
        If $(\Yb \independent S \vert \Xb)_{\Gr_{\overline{\Xb}}}$, then according to Rule 1 of do-calculus, we have
        \begin{equation} \label{eq: rule1-sub-pop}
            \pro_{\Xb}(\Yb \vert S = 1) = \pro_{\Xb}(\Yb).
        \end{equation}
        Therefore, the identifiability of $\pro_{\Xb}(\Yb)$ is equivalent to $\pro_{\Xb}(\Yb \vert S = 1)$, which concludes the proof.
    \end{proof}

\section{Numerical Experiment} \label{apd: exp}
   We conduct a numerical experiment to demonstrate the significance of the s-ID problem and assess the output of Algorithm \ref{algo: s-ID}. Note that the experiment is simple and can be run on a system with any level of computational power.

    Consider the following structural causal model (SCM) with the causal graph depicted in Figure \ref{fig: exp-experiment}.

    \begin{figure}[t]
            \centering
            \begin{tikzpicture}
            \tikzstyle{block-dashed} = [draw, fill=white, dashed, circle, text centered, inner sep=0.25cm]
                \node [block-dashed, label=center:$U_1$](U1) {};
                \node [block, left = 1 of U1, label=center:$X$](X) {};
                \node [block, right= 1 of U1, label=center:$Z$](Z) {};
             
                \node [block-dashed, below = 1.5 of U1, label=center:$U_2$](U2) {};    
                \node [block-s, right= 1 of U2, label=center:$S$](S) {};
                \node [block, left = 1 of U2, label=center:$Y$](Y) {};
                \draw[edge] (X) to (Y);
                \draw[edge] (Z) to (S); 
                \draw[edge, dashed] (U1) to (Z);
                \draw[edge, dashed] (U1) to (X);  
                \draw[edge, dashed] (U2) to (S);
                \draw[edge, dashed] (U2) to (Y);     
            \end{tikzpicture}
            \caption{A DAG, where $U_1$ and $U_2$ are unobservable, and $S$ represents the auxiliary variable for modeling a sub-population.}
            \label{fig: exp-experiment}
    \end{figure}
    \begin{align*}
        U_1 &\sim Bern(0.5) \\
        U_2 &\sim Bern(0.5) \\
        X  &= U_1 \oplus \varepsilon_x,~\varepsilon_x \sim Bern(0.2) \\
        Y  &= X \oplus U_2\\
        Z  &= U_1 \oplus \varepsilon_z,~\varepsilon_z \sim Bern(0.2)
    \end{align*}
    Herein, $\oplus$ denotes the XOR operation, all the variables $\{U_1, U_2, \varepsilon_x, \varepsilon_z,\varepsilon_{s_1}, \varepsilon_{s_2},  \varepsilon_{s_3} \}$ are independent, and $Bern(p)$ denotes a Bernoulli random variable with parameter $p$.
    Now, consider the sub-population with the following mechanism:
    \begin{equation*}
        S  = (Z \times \varepsilon_{s_1}) \oplus (U_2 \times \varepsilon_{s_2}) \oplus \varepsilon_{s_3}, (\varepsilon_{s_1},  \varepsilon_{s_2},  \varepsilon_{s_3}) \sim (Bern(0.6), Bern(0.9), Bern(0.1)) 
    \end{equation*}
    We now consider the problem of estimating the causal effect of $X$ on $Y$ in this sub-population.
    Particularly, our goal is to calculate $\pro_{X = 0}(Y = 1 \vert S = 1)$.
    
    \subsection{Theoretical Analysis}
        To analyze and compare the \s-ID and ID algorithms, we first determine the exact values of $\pro_{X=0}(Y=1 \vert S = 1)$ and $\pro_{X=0}(Y=1)$.
        According to the equation of $Y$ in the SCM, we have
        $$
            \pro_{X = x}(Y = y) = \pro_{X = x}(y = x \oplus U_2) = \pro(U_2 = y \oplus x).
        $$
        Since $U_2$ is a Bernoulli random variable with parameter 0.5, the above probability always equals 0.5. For $\pro_{X}(Y \vert S = 1)$, we have
        $$  
            \pro_{X = x}(Y = y \vert S = 1) = \pro_{X = x}(Y = x \oplus U_2 \vert S = 1) = \pro (U_2 = x \oplus y \vert S = 1).
        $$
        Therefore, for $X = 0$ and $Y = 1$, we need to compute $\pro (U_2 = 1 \vert S = 1)$.
        By using the equation of $S$ in the model,
        \begin{align*}
            \pro (U_2 = 1 \vert S = 1)  &= \frac{\pro (U_2 = 1, S = 1)}{\pro (U_2 = 1, S = 1) + \pro (U_2 = 0, S = 1)} \nonumber \\
                                        &= \frac{\pro (U_2 = 1) \pro(S = 1 \vert U_2 = 1)}{\pro (U_2 = 0) \pro(S = 1 \vert U_2 = 0) + \pro (U_2 = 1) \pro(S = 1 \vert U_2 = 1)} 
        \end{align*}
        Note that $\pro(U_2 = 0) = \pro(U_2 = 1) = 0.5$, hence, we have
        \begin{equation*}
            \pro (U_2 = 1 \vert S = 1) = \frac{\pro(S = 1 \vert U_2 = 1)}{\pro(S = 1 \vert U_2 = 0) + \pro(S = 1 \vert U_2 = 1)}
        \end{equation*}
        Since $\varepsilon_{s_1}$ and $\varepsilon_{s_3}$ and $Z$ are independent variables, let $W$ be $Z \times \varepsilon_{s_1} \oplus \varepsilon_{s_3}$; then, we have
        
        $$
            W \sim Bern(0.1 \times (1 - 0.6 \times 0.5)  + 0.9 \times 0.5 \times 0.6)  = Bern(0.34).
        $$
        Note that $W$ and $U_2 \times \varepsilon_{s_2}$ are independent, and $S = W + U_2 \times \varepsilon_{s_2}$; thus,
        
        \begin{align*}
            \pro(S = 1 \vert U_2 = 0) &= P(W = 1) = 0.34 \\
            \pro(S = 1 \vert U_2 = 1) &= P(W = 1)\pro(\varepsilon_{s_2} = 0) +  P(W = 0)\pro(\varepsilon_{s_2} = 1) \\
            &= 0.34 \times 0.1 + 0.66 \times 0.9 = 0.628
        \end{align*}
        Hence,
        \begin{equation} \label{eq: true}
            \pro_{X = 0}(Y = 1 \vert S = 1) = \frac{0.628}{0.628 + 0.34} \approx 0.648
        \end{equation}

    \subsection{Empirical Analysis}
         The ID algorithm returns the following simple expression for $\pro_{X}(Y)$
        \begin{equation}\label{eq: eq-id}
            \pro_{X}(Y) = \pro(Y \vert X).
        \end{equation}
        On the other hand, the s-ID algorithm returns
        \begin{equation} \label{eq: eq-sid}
            \pro_{X}(Y \vert S = 1) = \sum_{Z} \pro(Y \vert X, Z, S=1) \pro(Z \vert S=1). 
        \end{equation}
        Next, we generated 3000 samples from the population.
        We then computed $S$ for each generated sample and collected the samples where $S=1$, resulting in 1469 available samples from our target sub-population.
        Recall that our goal was to estimate $\pro_{X = 0}(Y = 1 \vert S = 1)$.
        Consider the following two approaches.
        \begin{itemize}
            \item If we consider the s-ID algorithm and the existence of the selection bias $S$, applying the simple plug-in estimator to the formula in \eqref{eq: eq-sid} results in
                \begin{align} \label{eq: sid est}
                     \pro_{X = 0}(Y = 1 \vert S = 1) \approx &\hat{\pro}(Y = 1 \vert X = 0, Z = 0) \hat{\pro}(Z = 0) \nonumber \\
                     &+ \hat{\pro}(Y = 1 \vert X = 0, Z = 1) \hat{\pro}(Z = 1)=0.641
                \end{align}
            \item
                Suppose we ignore the presence of sub-population and apply the ID algorithm. In this scenario, we have to estimate $\pro_{X}(Y)$ using the empirical distribution of variables, i.e., $\hat{\pro}(\Vb)$. If we use the ID algorithm, we need to estimate the quantity mentioned in equation \eqref{eq: eq-id}. The empirical estimate for this case is             
                \begin{equation} \label{eq: id est}
                    \pro_{X}(Y) \approx \hat{\pro}(Y \vert X) = 0.725.
                \end{equation}
        \end{itemize}

        A comparison between the estimation results of our proposed method in Equation \ref{eq: sid est} and the classical ID problem in Equation \eqref{eq: id est} with the true underlying value in Equation \eqref{eq: true} shows that our approach accurately computes the target causal effect. On the other hand, ignoring the subtleties related to sub-population can lead to erroneous estimation.


\newpage
\section*{NeurIPS Paper Checklist}

\begin{enumerate}

\item {\bf Claims}
    \item[] Question: Do the main claims made in the abstract and introduction accurately reflect the paper's contributions and scope?
    \item[] Answer: \answerYes{} 
    \item[] Justification: We provided detailed explanations for all claims, including the new graph structures and their properties in Section \ref{sec: s-ID}. In Section \ref{sec: main}, we presented our algorithm for solving the s-ID problem.
    \item[] Guidelines:
    \begin{itemize}
        \item The answer NA means that the abstract and introduction do not include the claims made in the paper.
        \item The abstract and/or introduction should clearly state the claims made, including the contributions made in the paper and important assumptions and limitations. A No or NA answer to this question will not be perceived well by the reviewers. 
        \item The claims made should match theoretical and experimental results, and reflect how much the results can be expected to generalize to other settings. 
        \item It is fine to include aspirational goals as motivation as long as it is clear that these goals are not attained by the paper. 
    \end{itemize}

\item {\bf Limitations}
    \item[] Question: Does the paper discuss the limitations of the work performed by the authors?
    \item[] Answer: \answerYes{} 
    \item[] Justification: We concretely discussed the problem setting in Sections \ref{sec: preliminaries} and \ref{sec: s-ID}. We also analyzed our proposed algorithm and its properties in Section \ref{sec: main}.
    \item[] Guidelines:
    \begin{itemize}
        \item The answer NA means that the paper has no limitation while the answer No means that the paper has limitations, but those are not discussed in the paper. 
        \item The authors are encouraged to create a separate "Limitations" section in their paper.
        \item The paper should point out any strong assumptions and how robust the results are to violations of these assumptions (e.g., independence assumptions, noiseless settings, model well-specification, asymptotic approximations only holding locally). The authors should reflect on how these assumptions might be violated in practice and what the implications would be.
        \item The authors should reflect on the scope of the claims made, e.g., if the approach was only tested on a few datasets or with a few runs. In general, empirical results often depend on implicit assumptions, which should be articulated.
        \item The authors should reflect on the factors that influence the performance of the approach. For example, a facial recognition algorithm may perform poorly when image resolution is low or images are taken in low lighting. Or a speech-to-text system might not be used reliably to provide closed captions for online lectures because it fails to handle technical jargon.
        \item The authors should discuss the computational efficiency of the proposed algorithms and how they scale with dataset size.
        \item If applicable, the authors should discuss possible limitations of their approach to address problems of privacy and fairness.
        \item While the authors might fear that complete honesty about limitations might be used by reviewers as grounds for rejection, a worse outcome might be that reviewers discover limitations that aren't acknowledged in the paper. The authors should use their best judgment and recognize that individual actions in favor of transparency play an important role in developing norms that preserve the integrity of the community. Reviewers will be specifically instructed to not penalize honesty concerning limitations.
    \end{itemize}

\item {\bf Theory Assumptions and Proofs}
    \item[] Question: For each theoretical result, does the paper provide the full set of assumptions and a complete (and correct) proof?
    \item[] Answer: \answerYes{} 
    \item[] Justification: We included all the assumptions in the main part of the paper. We also proved all our results in Appendices \ref{apd: lemmas} and \ref{apd: main}.
    \item[] Guidelines:
    \begin{itemize}
        \item The answer NA means that the paper does not include theoretical results. 
        \item All the theorems, formulas, and proofs in the paper should be numbered and cross-referenced.
        \item All assumptions should be clearly stated or referenced in the statement of any theorems.
        \item The proofs can either appear in the main paper or the supplemental material, but if they appear in the supplemental material, the authors are encouraged to provide a short proof sketch to provide intuition. 
        \item Inversely, any informal proof provided in the core of the paper should be complemented by formal proofs provided in appendix or supplemental material.
        \item Theorems and Lemmas that the proof relies upon should be properly referenced. 
    \end{itemize}

    \item {\bf Experimental Result Reproducibility}
    \item[] Question: Does the paper fully disclose all the information needed to reproduce the main experimental results of the paper to the extent that it affects the main claims and/or conclusions of the paper (regardless of whether the code and data are provided or not)?
    \item[] Answer: \answerYes{} 
    \item[] Justification: All steps of the conducted experiment are detailed in Appendix \ref{apd: exp}. We utilize synthetic data corresponding to a causal model defined in Appendix \ref{apd: exp}. Hence, the data generation process is straightforward. After that, we only need some estimation of certain distributions from this data to reproduce the results.
    \item[] Guidelines:
    \begin{itemize}
        \item The answer NA means that the paper does not include experiments.
        \item If the paper includes experiments, a No answer to this question will not be perceived well by the reviewers: Making the paper reproducible is important, regardless of whether the code and data are provided or not.
        \item If the contribution is a dataset and/or model, the authors should describe the steps taken to make their results reproducible or verifiable. 
        \item Depending on the contribution, reproducibility can be accomplished in various ways. For example, if the contribution is a novel architecture, describing the architecture fully might suffice, or if the contribution is a specific model and empirical evaluation, it may be necessary to either make it possible for others to replicate the model with the same dataset, or provide access to the model. In general. releasing code and data is often one good way to accomplish this, but reproducibility can also be provided via detailed instructions for how to replicate the results, access to a hosted model (e.g., in the case of a large language model), releasing of a model checkpoint, or other means that are appropriate to the research performed.
        \item While NeurIPS does not require releasing code, the conference does require all submissions to provide some reasonable avenue for reproducibility, which may depend on the nature of the contribution. For example
        \begin{enumerate}
            \item If the contribution is primarily a new algorithm, the paper should make it clear how to reproduce that algorithm.
            \item If the contribution is primarily a new model architecture, the paper should describe the architecture clearly and fully.
            \item If the contribution is a new model (e.g., a large language model), then there should either be a way to access this model for reproducing the results or a way to reproduce the model (e.g., with an open-source dataset or instructions for how to construct the dataset).
            \item We recognize that reproducibility may be tricky in some cases, in which case authors are welcome to describe the particular way they provide for reproducibility. In the case of closed-source models, it may be that access to the model is limited in some way (e.g., to registered users), but it should be possible for other researchers to have some path to reproducing or verifying the results.
        \end{enumerate}
    \end{itemize}

\item {\bf Open access to data and code}
    \item[] Question: Does the paper provide open access to the data and code, with sufficient instructions to faithfully reproduce the main experimental results, as described in supplemental material?
    \item[] Answer: \answerNo{} 
    \item[] Justification: As discussed in the previous question, our numerical experiment consists of a few simple steps. One can generate data according to the defined causal model in our experiment and reproduce the results. Hence, we have not provided the code.
    \item[] Guidelines:
    \begin{itemize}
        \item The answer NA means that paper does not include experiments requiring code.
        \item Please see the NeurIPS code and data submission guidelines (\url{https://nips.cc/public/guides/CodeSubmissionPolicy}) for more details.
        \item While we encourage the release of code and data, we understand that this might not be possible, so “No” is an acceptable answer. Papers cannot be rejected simply for not including code, unless this is central to the contribution (e.g., for a new open-source benchmark).
        \item The instructions should contain the exact command and environment needed to run to reproduce the results. See the NeurIPS code and data submission guidelines (\url{https://nips.cc/public/guides/CodeSubmissionPolicy}) for more details.
        \item The authors should provide instructions on data access and preparation, including how to access the raw data, preprocessed data, intermediate data, and generated data, etc.
        \item The authors should provide scripts to reproduce all experimental results for the new proposed method and baselines. If only a subset of experiments are reproducible, they should state which ones are omitted from the script and why.
        \item At submission time, to preserve anonymity, the authors should release anonymized versions (if applicable).
        \item Providing as much information as possible in supplemental material (appended to the paper) is recommended, but including URLs to data and code is permitted.
    \end{itemize}

\item {\bf Experimental Setting/Details}
    \item[] Question: Does the paper specify all the training and test details (e.g., data splits, hyperparameters, how they were chosen, type of optimizer, etc.) necessary to understand the results?
    \item[] Answer: \answerYes{} 
    \item[] Justification: We thoroughly elaborated on all the steps of our experiment in Appendix \ref{apd: exp}.
    \item[] Guidelines:
    \begin{itemize}
        \item The answer NA means that the paper does not include experiments.
        \item The experimental setting should be presented in the core of the paper to a level of detail that is necessary to appreciate the results and make sense of them.
        \item The full details can be provided either with the code, in appendix, or as supplemental material.
    \end{itemize}

\item {\bf Experiment Statistical Significance}
    \item[] Question: Does the paper report error bars suitably and correctly defined or other appropriate information about the statistical significance of the experiments?
    \item[] Answer: \answerYes{} 
    \item[] Justification: Our experiment aims to compare the outcomes of our algorithm and the ID algorithm, which provide estimands for causal effects. We focus on scenarios with an infinite sample size, eliminating the need to report error rates. Further details can be found in Appendix \ref{apd: exp}.
    \item[] Guidelines:
    \begin{itemize}
        \item The answer NA means that the paper does not include experiments.
        \item The authors should answer "Yes" if the results are accompanied by error bars, confidence intervals, or statistical significance tests, at least for the experiments that support the main claims of the paper.
        \item The factors of variability that the error bars are capturing should be clearly stated (for example, train/test split, initialization, random drawing of some parameter, or overall run with given experimental conditions).
        \item The method for calculating the error bars should be explained (closed form formula, call to a library function, bootstrap, etc.)
        \item The assumptions made should be given (e.g., Normally distributed errors).
        \item It should be clear whether the error bar is the standard deviation or the standard error of the mean.
        \item It is OK to report 1-sigma error bars, but one should state it. The authors should preferably report a 2-sigma error bar than state that they have a 96\% CI, if the hypothesis of Normality of errors is not verified.
        \item For asymmetric distributions, the authors should be careful not to show in tables or figures symmetric error bars that would yield results that are out of range (e.g. negative error rates).
        \item If error bars are reported in tables or plots, The authors should explain in the text how they were calculated and reference the corresponding figures or tables in the text.
    \end{itemize}

\item {\bf Experiments Compute Resources}
    \item[] Question: For each experiment, does the paper provide sufficient information on the computer resources (type of compute workers, memory, time of execution) needed to reproduce the experiments?
    \item[] Answer: \answerYes{} 
    \item[] Justification: In Appendix \ref{apd: exp}, we note that our numerical experiment does not need specific computational power.
    \item[] Guidelines:
    \begin{itemize}
        \item The answer NA means that the paper does not include experiments.
        \item The paper should indicate the type of compute workers CPU or GPU, internal cluster, or cloud provider, including relevant memory and storage.
        \item The paper should provide the amount of compute required for each of the individual experimental runs as well as estimate the total compute. 
        \item The paper should disclose whether the full research project required more compute than the experiments reported in the paper (e.g., preliminary or failed experiments that didn't make it into the paper). 
    \end{itemize}
    
\item {\bf Code Of Ethics}
    \item[] Question: Does the research conducted in the paper conform, in every respect, with the NeurIPS Code of Ethics \url{https://neurips.cc/public/EthicsGuidelines}?
    \item[] Answer: \answerYes{} 
    \item[] Justification: The research conducted in the paper adheres fully to the NeurIPS Code of Ethics.
    \item[] Guidelines:
    \begin{itemize}
        \item The answer NA means that the authors have not reviewed the NeurIPS Code of Ethics.
        \item If the authors answer No, they should explain the special circumstances that require a deviation from the Code of Ethics.
        \item The authors should make sure to preserve anonymity (e.g., if there is a special consideration due to laws or regulations in their jurisdiction).
    \end{itemize}

\item {\bf Broader Impacts}
    \item[] Question: Does the paper discuss both potential positive societal impacts and negative societal impacts of the work performed?
    \item[] Answer: \answerNo{}{}{} 
    \item[] Justification: We think there is no societal impact as our work primarily focuses on theoretical analysis of a problem in causal effect identification.
    \item[] Guidelines:
    \begin{itemize}
        \item The answer NA means that there is no societal impact of the work performed.
        \item If the authors answer NA or No, they should explain why their work has no societal impact or why the paper does not address societal impact.
        \item Examples of negative societal impacts include potential malicious or unintended uses (e.g., disinformation, generating fake profiles, surveillance), fairness considerations (e.g., deployment of technologies that could make decisions that unfairly impact specific groups), privacy considerations, and security considerations.
        \item The conference expects that many papers will be foundational research and not tied to particular applications, let alone deployments. However, if there is a direct path to any negative applications, the authors should point it out. For example, it is legitimate to point out that an improvement in the quality of generative models could be used to generate deepfakes for disinformation. On the other hand, it is not needed to point out that a generic algorithm for optimizing neural networks could enable people to train models that generate Deepfakes faster.
        \item The authors should consider possible harms that could arise when the technology is being used as intended and functioning correctly, harms that could arise when the technology is being used as intended but gives incorrect results, and harms following from (intentional or unintentional) misuse of the technology.
        \item If there are negative societal impacts, the authors could also discuss possible mitigation strategies (e.g., gated release of models, providing defenses in addition to attacks, mechanisms for monitoring misuse, mechanisms to monitor how a system learns from feedback over time, improving the efficiency and accessibility of ML).
    \end{itemize}
    
\item {\bf Safeguards}
    \item[] Question: Does the paper describe safeguards that have been put in place for responsible release of data or models that have a high risk for misuse (e.g., pretrained language models, image generators, or scraped datasets)?
    \item[] Answer: \answerNA{} 
    \item[] Justification: The focus of our paper is the theoretical understanding of a specific problem in the area of causal inference. Therefore, this question is not applicable.
    \item[] Guidelines:
    \begin{itemize}
        \item The answer NA means that the paper poses no such risks.
        \item Released models that have a high risk for misuse or dual-use should be released with necessary safeguards to allow for controlled use of the model, for example by requiring that users adhere to usage guidelines or restrictions to access the model or implementing safety filters. 
        \item Datasets that have been scraped from the Internet could pose safety risks. The authors should describe how they avoided releasing unsafe images.
        \item We recognize that providing effective safeguards is challenging, and many papers do not require this, but we encourage authors to take this into account and make a best faith effort.
    \end{itemize}

\item {\bf Licenses for existing assets}
    \item[] Question: Are the creators or original owners of assets (e.g., code, data, models), used in the paper, properly credited and are the license and terms of use explicitly mentioned and properly respected?
    \item[] Answer: \answerNA{} 
    \item[] Justification: The focus of our paper is the theoretical understanding of a specific problem in the area of causal inference. Therefore, this question is not applicable.
    \item[] Guidelines:
    \begin{itemize}
        \item The answer NA means that the paper does not use existing assets.
        \item The authors should cite the original paper that produced the code package or dataset.
        \item The authors should state which version of the asset is used and, if possible, include a URL.
        \item The name of the license (e.g., CC-BY 4.0) should be included for each asset.
        \item For scraped data from a particular source (e.g., website), the copyright and terms of service of that source should be provided.
        \item If assets are released, the license, copyright information, and terms of use in the package should be provided. For popular datasets, \url{paperswithcode.com/datasets} has curated licenses for some datasets. Their licensing guide can help determine the license of a dataset.
        \item For existing datasets that are re-packaged, both the original license and the license of the derived asset (if it has changed) should be provided.
        \item If this information is not available online, the authors are encouraged to reach out to the asset's creators.
    \end{itemize}

\item {\bf New Assets}
    \item[] Question: Are new assets introduced in the paper well documented and is the documentation provided alongside the assets?
    \item[] Answer: \answerNA{} 
    \item[] Justification: The focus of our paper is the theoretical understanding of a specific problem in the area of causal inference. Therefore, this question is not applicable.
    \item[] Guidelines:
    \begin{itemize}
        \item The answer NA means that the paper does not release new assets.
        \item Researchers should communicate the details of the dataset/code/model as part of their submissions via structured templates. This includes details about training, license, limitations, etc. 
        \item The paper should discuss whether and how consent was obtained from people whose asset is used.
        \item At submission time, remember to anonymize your assets (if applicable). You can either create an anonymized URL or include an anonymized zip file.
    \end{itemize}

\item {\bf Crowdsourcing and Research with Human Subjects}
    \item[] Question: For crowdsourcing experiments and research with human subjects, does the paper include the full text of instructions given to participants and screenshots, if applicable, as well as details about compensation (if any)? 
    \item[] Answer: \answerNA{} 
    \item[] Justification: The focus of our paper is the theoretical understanding of a specific problem in areas of causal inference. Therefore, this question is not applicable.
    \item[] Guidelines:
    \begin{itemize}
        \item The answer NA means that the paper does not involve crowdsourcing nor research with human subjects.
        \item Including this information in the supplemental material is fine, but if the main contribution of the paper involves human subjects, then as much detail as possible should be included in the main paper. 
        \item According to the NeurIPS Code of Ethics, workers involved in data collection, curation, or other labor should be paid at least the minimum wage in the country of the data collector. 
    \end{itemize}

\item {\bf Institutional Review Board (IRB) Approvals or Equivalent for Research with Human Subjects}
    \item[] Question: Does the paper describe potential risks incurred by study participants, whether such risks were disclosed to the subjects, and whether Institutional Review Board (IRB) approvals (or an equivalent approval/review based on the requirements of your country or institution) were obtained?
    \item[] Answer: \answerNA{} 
    \item[] Justification: The focus of our paper is the theoretical understanding of a specific problem in the area of causal inference. Therefore, this question is not applicable.
    \item[] Guidelines:
    \begin{itemize}
        \item The answer NA means that the paper does not involve crowdsourcing nor research with human subjects.
        \item Depending on the country in which research is conducted, IRB approval (or equivalent) may be required for any human subjects research. If you obtained IRB approval, you should clearly state this in the paper. 
        \item We recognize that the procedures for this may vary significantly between institutions and locations, and we expect authors to adhere to the NeurIPS Code of Ethics and the guidelines for their institution. 
        \item For initial submissions, do not include any information that would break anonymity (if applicable), such as the institution conducting the review.
    \end{itemize}

\end{enumerate}

\end{document}